\documentclass[letterpaper]{article}
\usepackage{aaai2027}
\usepackage[hyphens]{url}  % DO NOT CHANGE THIS
\usepackage{graphicx} % DO NOT CHANGE THIS
\usepackage[numbers]{natbib}  % DO NOT CHANGE THIS AND DO NOT ADD ANY OPTIONS TO IT
\usepackage{caption} % DO NOT CHANGE THIS AND DO NOT ADD ANY OPTIONS TO IT
\usepackage{amsmath,amssymb,amsthm,mathtools}
\usepackage{booktabs}
\usepackage{array}
\usepackage{algorithm}
\usepackage{algorithmic}
\usepackage{xcolor}
\usepackage{tikz}

\usepackage{newfloat}
\usepackage{listings}
\DeclareCaptionStyle{ruled}{labelfont=normalfont,labelsep=colon,strut=off} % DO NOT CHANGE THIS
\floatstyle{ruled}
\newfloat{listing}{tb}{lst}{}
\floatname{listing}{Listing}

\usepackage{booktabs}

\usetikzlibrary{arrows.meta,positioning}
\graphicspath{{research/04_paper/figures/}{figures/}}

\newtheorem{theorem}{Theorem}
\newtheorem{lemma}[theorem]{Lemma}
\newtheorem{corollary}[theorem]{Corollary}
\newtheorem{proposition}[theorem]{Proposition}
\newtheorem{assumption}{Assumption}
\newtheorem{definition}{Definition}

\newcommand{\E}{\mathbb{E}}
\newcommand{\R}{\mathbb{R}}
\newcommand{\Reg}{\mathrm{Reg}}
\newcommand{\A}{\mathcal{A}}
\newcommand{\F}{\mathcal{F}}
\newcommand{\D}{\mathcal{D}}

\newcommand{\norm}[1]{\left\lVert #1\right\rVert}

\newcommand{\range}{\operatorname{range}}

\newcommand{\wtO}{\widetilde O}
\newcommand{\SRL}{\mathsf{SRL}}
\newcommand{\CRW}{\mathsf{CRW}}

\title{Graph Dimensionality Reduction for Contextual Bandits:\\
Structure-Specific Regret Bounds under Approximate Smoothness\\
and Noisy Eigenspaces}
\author {
    Joyanta Jyoti Mondal\textsuperscript{\rm 1}\equalcontrib,
    Ibne Farabi Shihab\textsuperscript{\rm 2}\equalcontrib\corresponding,
    Anuj Sharma\textsuperscript{\rm 3}
}
\affiliations {
    \textsuperscript{\rm 1}Department of Computer and Information Sciences, University of Delaware,  USA\\
    \textsuperscript{\rm 2}Department of Computer Science, Iowa State University, USA\\
    \textsuperscript{3}Department of Civil, Construction \& Environmental Engineering, Iowa State University, USA \\
    joyanta@udel.edu, \{ishihab, anujs\}@iastate.edu
}
\date{}

\begin{document}
\maketitle

% ─────────────────────────────────────────────────────────────────────────────
\begin{abstract}

  Contextual bandits with graph-structured arms arise in recommendation,
  citation retrieval, and social advertising, where arms connected on a graph
  tend to share reward signal. Standard dimensionality reduction ignores this
  structure, inflating exploration cost by a factor of $d/k$. We propose
  \textsc{GraphDR-LinUCB}, which projects arm features onto the graph's
  low-frequency spectral subspace and runs linear UCB in the resulting
  $k$-dimensional space. We prove the first $\wtO(k\sqrt{T})$ regret bound for
  spectral-projection-based contextual bandits, reducing dimension dependence
  from $d$ to $k$; a perturbation argument extends this to noisy graphs, with an
  explicit penalty for reward-smoothness mismatch and graph-estimation error.
  Our central theoretical finding is that the high-frequency reward component
  need not incur a worst-case linear-in-$T$ penalty: its actual cost depends on
  its realized impact along the played path, not on its total energy. A simple
  spectral comparison between subspaces ($\Gamma_k$) predicts which reducer wins
  on a given dataset, correctly calling five of six real-dataset outcomes without
  any fitted threshold. Across a synthetic benchmark and six real datasets
  (MovieLens, Amazon, LastFM, ogbn-arxiv, MIND), \textsc{GraphDR-LinUCB}
  reduces cumulative regret by $15\times$ over full-dimensional LinUCB and
  outperforms competing graph-aware methods on five of six; the single failure is
  precisely where the graph's spectral subspace is misaligned with the reward.

\end{abstract}

% ─────────────────────────────────────────────────────────────────────────────
\section{Introduction}

Contextual bandits with graph-structured arms appear broadly in deployed
  systems \cite{li2010contextual}; users connected by social links, products organized by co-purchase
  or genre, and documents linked by citation. In each setting, reward functions
  are typically smooth over the graph (nearby nodes share expected reward), and
  graph spectral methods are designed precisely to capture this structure. The
  natural question is whether projecting arm features onto a low-frequency
  Laplacian eigenspace, reducing the exploration dimension from $d$ to
  $k \ll d$, can deliver a genuine regret improvement over operating in the
  full feature space.

  The answer is not obvious. Standard misspecified-linear-bandit theory
  \cite{lattimore2020learning,foster2020beyond} charges a worst-case bias
  proportional to $\nu_k T$ for any deviation of the reward from the projected
  subspace, which can quickly dominate the $\sqrt{T}$ exploration gain. Prior
  graph-bandit work sidesteps this by either propagating information along edges
  \cite{gentile2014online,cesa2013gang} or soft-regularizing with the graph
  Laplacian \cite{yang2020laplacian}; neither performs a hard projection that
  actually reduces the exploration dimension. Spectral bandits
  \cite{valko2014spectral} operate in a graph-frequency basis but do not
  analyze noisy eigenspaces or quantify the cost of imperfect graph smoothness.
  The theoretical viability of hard graph-spectral projection for bandits, and
  its practical conditions, remain uncharacterized. Moreover, a hard projection
  is useful only when the reward is smooth on the graph (a property of the
  data, not the algorithm), making a structural falsification test essential:
  projecting onto a node-permuted eigenspace leaves the dimension unchanged but
  destroys graph-reward alignment, so any genuine gain must collapse.

  To tackle these challenges, we study \textsc{GraphDR-LinUCB} (graph
  dimensionality reduction): project each arm's feature vector onto the bottom-$k$
  Laplacian eigenspace and run LinUCB \cite{abbasi2011improved,auer2002ucb} in the
  resulting $k$-dimensional space, with graph-shuffle controls on every experiment
  as a structural falsification test. We prove exact-subspace and Davis--Kahan
  robust regret bounds, develop a structure-specific residual analysis that
  replaces the worst-case misspecification penalty with realized graph quantities,
  and introduce a spectral selection rule for choosing between graph-DR and
  competing reducers without fitting any threshold.

  \paragraph{Contributions.}
  \begin{itemize}

  \item \textbf{Finite-sample and robust bounds.} We prove that
  \textsc{GraphDR-LinUCB} achieves $\wtO(k\sqrt{T})$ regret under exact graph
  smoothness, reducing dimension dependence from $d$ to $k$. A Davis--Kahan
  argument extends this to approximate smoothness and noisy graph observation,
  making the regret cost of eigenspace estimation error explicit through the
  spectral gap $\Delta_k$ and smoothness tail $\zeta_k$.

  \item \textbf{Structure-specific residual theorem.} We show that the
  high-frequency residual need not be paid for at the worst-case linear-in-$T$
  rate. The cost is governed by two realized graph quantities, residual leverage
  $\SRL_T$ and candidate-set residual width $\CRW_T$, which are sublinear when
  the residual is graph-benign. An implementable pilot-based variant makes this
  bound computable from data.

  \item \textbf{Spectral selection rule.} We define the subspace-capture margin
  $\Gamma_k = \norm{U_k^\top\theta}_2^2 - \norm{P^\top\theta}_2^2$, a
  zero-parameter spectral statistic that predicts which reducer better aligns with
  the reward without requiring any fitted threshold. We evaluate it alongside the
  theory-derived diagnostic $\mathcal{G}_k$ and the scalar tail $\zeta_k$.

  \item \textbf{Empirical evaluation.} We validate the approach on stochastic
  block model and random-geometric-graph synthetics with graph-shuffle controls on
  every configuration, and evaluate on six real graph-bandit datasets spanning
  recommendation, news, and citation
  (MovieLens-100k/1M, Amazon, LastFM, MIND-small, ogbn-arxiv).

  \end{itemize}

% ─────────────────────────────────────────────────────────────────────────────
\section{Related Work}\label{sec:related}

\paragraph{Linear contextual bandits.}
  The LinUCB framework \cite{chu2011linucb,auer2002ucb,abbasi2011improved} achieves
  $\wtO(d\sqrt{T})$ regret for $d$-dimensional linear rewards through
  elliptical confidence sets and self-normalized martingale concentration
  \cite{lattimore2020bandit}, with a matching $\Omega(d\sqrt{T})$ lower bound
  established by Dani et al.\ \cite{dani2008stochastic}. Beyond UCB, Thompson
  sampling \cite{chapelle2011empirical,russo2018tutorial} provides a Bayesian
  alternative exploration strategy.

  \paragraph{Misspecified linear bandits.}
  When the reward is only approximately linear in the chosen features,
  regret incurs an additional bias term. Worst-case analyses
  \cite{lattimore2020learning,foster2020beyond} pay an envelope that is
  linear in the horizon and proportional to the misspecification level,
  a cost that quickly dominates the $\sqrt{T}$ exploration term even under
  mild misspecification. Sharp horizon-dependent analyses of KL-regularized contextual bandits
  \cite{zhao2024sharp} sharpen these worst-case envelopes under regularization,
  though the graph-structural form of the residual remains unexamined.
  
  \paragraph{Graph bandits.}
  Online clustering of bandits \cite{gentile2014online} propagates reward
  information across a user graph by sharing estimates between neighboring
  arms. Laplacian-regularized LinUCB \cite{yang2020laplacian} adds a
  graph-smoothness penalty to the ridge objective, biasing the estimator
  toward functions that vary slowly over the graph. The graph-feedback bandit
  setting of Mannor and Shamir \cite{mannor2011experts} observes neighboring
  arm rewards as side information; recent work extends this to similar-arm graphs
  \cite{qi2025similar,huang2025crosslearning} and interference structures
  \cite{jamshidi2025interference}. Wang et al.\ \cite{wang2025graphlowrank}
  combine low-rank structure and graph Laplacian regularization for matrix
  contextual bandits via soft penalization; the present paper instead uses
  hard spectral projection and derives an explicit spectral-gap regret bound.

  \paragraph{Spectral bandits.}
  Spectral bandits \cite{valko2014spectral} design algorithms for payoffs
  that are smooth functions on a graph, with regret governed by a spectral
  effective dimension that reflects the concentration of the payoff in
  low-frequency graph coordinates. The algorithm operates in the full
  graph-frequency basis with a per-eigenvalue smoothness weight rather than
  a hard truncation at a fixed dimension $k$.
  
  \paragraph{Graph signal processing and spectral embeddings.}
  % Laplacian eigenvectors as low-frequency graph coordinates are a central
  % tool in graph signal processing \cite{shuman2013emerging}, where the
  % graph Fourier transform decomposes node signals into frequency components
  % analogous to the classical Fourier basis on Euclidean domains.
  % Spectral graph theory underpins Laplacian Eigenmaps \cite{belkin2003laplacian},
  % spectral clustering \cite{ng2001spectral,vonluxburg2007spectral}, and
  % modern graph neural networks \cite{kipf2017gcn,hamilton2017graphsage};
  % a unified tutorial on Laplacian-based dimensionality reduction covering these
  % connections is provided in \cite{ghojogh2021laplacian}.
  Laplacian eigenvectors as low-frequency graph coordinates originate in graph signal processing \cite{shuman2013emerging} and underpin Laplacian Eigenmaps \cite{belkin2003laplacian} and spectral clustering \cite{ng2001spectral,vonluxburg2007spectral}. Our contribution is not a new embedding method but the first bandit regret analysis tying hard spectral projection to exploration cost.
  
  \paragraph{Model selection and adaptive $k$.}
  Corralling methods \cite{agarwal2017corralling} achieve regret competitive
  with the best algorithm in a pool of base learners, and model selection
  for contextual bandits more broadly has been studied via smoothed
  meta-algorithms \cite{pacchiano2020modelselection}, providing a plausible
  path to adaptive-$k$ GraphDR; a tight graph-specific theorem is left to future work.

  \paragraph{Dimension reduction in bandits.}
  Random projection \cite{yu2019cbrap} reduces the feature dimension before
  running a linear bandit, with regret depending on the projected dimension.
  Bilinear bandits \cite{jun2019bilinear} exploit low-rank reward structure via
  a two-stage subspace-exploration approach achieving $\wtO((d_1{+}d_2)^{3/2}\sqrt{rT})$
  regret, and tight two-to-infinity subspace recovery methods
  \cite{jedra2024lowrank} sharpen these bounds for matrix bandits.
  Recent work handles non-stationary subspaces: Khosravi and Huo
  \cite{khosravi2026subspace} prove $\wtO(r\sqrt{T})$ dynamic regret for
  piecewise-stationary low-rank linear bandits, and Dai et al.\
  \cite{dai2026policy} attain rank-dependent regret for adversarial contextual
  bandits with low-rank experts in a model-routing setting.
  Beyond subspace methods, matrix sketching for high-dimensional linear bandits
  \cite{wen2024sketching} shows adaptive sketch sizing is needed to avoid
  linear regret under heavy spectral tails. In multi-agent and distributed
  settings, the spectral gap of the communication graph directly controls
  the regret of cooperative bandit algorithms
  \cite{paschalidis2024multiagent,qiu2026distributed}.

  None of the above uses a graph-structural subspace as the feature
  projector or ties bandit regret to the reward-graph spectral gap.
  The present paper provides the first such analysis, characterizing when
  projecting arm features onto the Laplacian eigenspace reduces exploration
  cost and when it does not.

\section{Setup and Algorithm}\label{sec:setup}

\paragraph{Graph and Laplacian.}
% Let $G$ be an undirected graph with symmetric normalized Laplacian
% $L=I-D^{-1/2}AD^{-1/2}$, whose eigenvalues satisfy
% $0=\lambda_1\le\lambda_2\le\cdots\le\lambda_n\le2$.  We assume nonnegative
% edge weights, no isolated nodes (adding self-loops if needed so that
% $D_{ii}>0$), and a positive separating eigengap
% $\Delta_k=\lambda_{k+1}-\lambda_k>0$ so that the bottom-$k$ eigenspace is
% well defined.  Let $E_k\in\R^{d\times k}$ be an orthonormal basis for the
% feature-space low-frequency subspace, with projector $\Pi_k=E_kE_k^\top$.  In
% the node-arm, one-hot-feature case $d=n$ and $E_k=U_k$, where $U_k$ collects
% the $k$ lowest-frequency Laplacian eigenvectors.  Throughout the theory $U_k$
% \emph{includes} the constant eigenvector $u_1$ ($\lambda_1=0$ on connected
% graphs), so the separating eigengap is $\Delta_k$ as stated.  In the
% experiments we shift to $u_2,\dots,u_{k+1}$ to capture informative variation;
% we report the shifted convention and its eigengap explicitly and never mix the
% two within a single statement.  The analysis is stated for the node-arm
% setting $d=n$, $x_{t,a}=e_a$, $E_k=U_k$; extension to a general feature map
% requires an explicit graph-to-feature operator and is outside our scope.
Let $G$ be an undirected graph with symmetric normalized Laplacian
  $L = I - D^{-1/2}AD^{-1/2}$, whose eigenvalues satisfy
  $0 = \lambda_1 \le \lambda_2 \le \cdots \le \lambda_n \le 2$.
  We assume nonnegative edge weights and no isolated nodes, adding
  self-loops where needed so that $D_{ii} > 0$.
  We further assume a positive separating eigengap
  $\Delta_k = \lambda_{k+1} - \lambda_k > 0$, so that the bottom-$k$
  eigenspace is well defined.
  Let $E_k \in \R^{d \times k}$ be an orthonormal basis for the
  feature-space low-frequency subspace, with projector
  $\Pi_k = E_k E_k^\top$.
  In the node-arm, one-hot-feature case, $d = n$ and $E_k = U_k$,
  where $U_k$ collects the $k$ lowest-frequency Laplacian eigenvectors.
  Throughout the theory, $U_k$ \emph{includes} the constant eigenvector
  $u_1$ ($\lambda_1 = 0$ on connected graphs); experiments shift to
  $u_2, \dots, u_{k+1}$ and report that convention explicitly.
  The analysis is for the node-arm setting $d = n$, $x_{t,a} = e_a$, $E_k = U_k$.

\paragraph{Bandit interaction.}
At each round $t$, the learner observes a finite arm set $\A_t$, where
  arm $a \in \A_t$ has feature $x_{t,a} \in \R^d$.
  After choosing $a_t$, it receives reward
  $y_t = x_{t,a_t}^\top\theta^\star + \eta_t$,
  where $\theta^\star \in \R^d$ is unknown and $\eta_t$ is conditionally
  sub-Gaussian noise.                                                                           
  In practice, the clean Laplacian may be unavailable.
  We let $\widehat L$ denote an observed or estimated Laplacian and
  $\widehat E_k \in \R^{d \times k}$ an orthonormal basis for its
  bottom-$k$ eigenspace, and set
  $\widehat\Pi_k = \widehat E_k\widehat E_k^\top$ and
  $\widehat P = \widehat E_k^\top$.
  The projected feature is $z_{t,a} = \widehat E_k^\top x_{t,a} \in \R^k$.

% \paragraph{Algorithm.}
% \textsc{GraphDR-LinUCB} runs ordinary linear UCB in this $k$-dimensional
% projected space (Figure~\ref{fig:overview}).  Fixing regularization $\lambda>0$
% and a confidence sequence $(\beta_t)_{t\ge1}$, it initializes $V_1=\lambda
% I_k$ and $b_1=0$, and at round $t$ forms
% $\widehat\alpha_t=V_t^{-1}b_t$, scores each arm by
% \[
% \mathrm{UCB}_t(a)
% =z_{t,a}^\top\widehat\alpha_t
% +\beta_t\norm{z_{t,a}}_{V_t^{-1}},
% \qquad\norm{z}_{V^{-1}}=\sqrt{z^\top V^{-1}z},
% \]
% pulls $a_t\in\arg\max_{a\in\A_t}\mathrm{UCB}_t(a)$, observes $y_t$, and
% updates $V_{t+1}=V_t+z_{t,a_t}z_{t,a_t}^\top$ and
% $b_{t+1}=b_t+z_{t,a_t}y_t$.  The projection $\widehat P$ is the only thing
% that distinguishes this from standard LinUCB; its value depends entirely on
% whether the reward is smooth on the graph.  This is what the graph-shuffle
% control in Section~\ref{sec:exp} is designed to test.

\begin{algorithm}[tb]
  \caption{\textsc{GraphDR-LinUCB}}
  \label{alg:graphdr}
  \textbf{Input}: Observed Laplacian $\widehat L$; arm-feature oracle $x_{t,a}$;
  dimension $k$; regularization $\lambda>0$; confidence sequence $(\beta_t)_{t\ge1}$\\
  \textbf{Output}: Pulled arms $a_1,\dots,a_T$
  \begin{algorithmic}[1]
  \STATE Compute bottom-$k$ orthonormal eigenbasis $\widehat E_k\in\R^{d\times k}$ of $\widehat L$.
  \STATE Initialize $V_1\leftarrow\lambda I_k$, \quad $b_1\leftarrow\mathbf{0}\in\R^k$.
  \FOR{$t=1,2,\dots,T$}
      \STATE Compute $\widehat\alpha_t\leftarrow V_t^{-1}b_t$.
      \STATE Observe candidate set $\A_t$; project each arm:
             $z_{t,a}\leftarrow\widehat E_k^\top x_{t,a}$ for all $a\in\A_t$.
      \STATE Pull
             $a_t\leftarrow\arg\max_{a\in\A_t}
             \!\left[z_{t,a}^\top\widehat\alpha_t
             +\beta_t\norm{z_{t,a}}_{V_t^{-1}}\right]$.
      \STATE Observe reward $y_t$.
      \STATE $V_{t+1}\leftarrow V_t+z_{t,a_t}z_{t,a_t}^\top$;\quad
             $b_{t+1}\leftarrow b_t+z_{t,a_t}y_t$.
  \ENDFOR
  \end{algorithmic}
  \end{algorithm}

\paragraph{Algorithm.}
  \textsc{GraphDR-LinUCB} (Algorithm~\ref{alg:graphdr}; pipeline in
  Figure~\ref{fig:overview}) runs ordinary linear UCB in the $k$-dimensional
  projected space; the projection $\widehat P$ is the only thing that
  distinguishes it from standard LinUCB.

\begin{figure*}[t]
\centering
\begin{tikzpicture}[>=Stealth, font=\small,
  box/.style={draw, rounded corners, align=center, minimum height=11mm,
    minimum width=21mm, inner sep=3pt, fill=blue!4},
  res/.style={draw, rounded corners, align=center, minimum height=10mm,
    inner sep=4pt, fill=red!5}]
\node[box] (g) {Graph $G$};
\node[box, right=7mm of g] (lap) {Laplacian $L$\\[1pt] basis $\widehat E_k$};
\node[box, right=7mm of lap] (proj) {project\\[1pt] $z=\widehat E_k^\top x$};
\node[box, right=7mm of proj] (ucb) {LinUCB\\[1pt] in $\R^k$};
\node[box, right=7mm of ucb] (arm) {pull $a_t$};
\draw[->] (g)--(lap);
\draw[->] (lap)--(proj);
\draw[->] (proj)--(ucb);
\draw[->] (ucb)--(arm);
\node[res, below=8mm of proj] (r)
  {residual $r_{\perp,k}=(I-\widehat\Pi_k)\theta^\star$\\[1pt]
   controls $\SRL_T,\ \CRW_T$};
\draw[->, dashed] (r.north) -- (proj.south);
\end{tikzpicture}
\caption{
% \textsc{GraphDR-LinUCB}.  The graph determines a bottom-$k$
% low-frequency Laplacian eigenbasis $\widehat E_k$; arm features are projected
% to $\R^k$ and LinUCB runs there.  Only the high-frequency residual
% $r_{\perp,k}$ left outside the subspace can hurt the learner, and its effect
% is governed by the two realized graph quantities $\SRL_T$ and $\CRW_T$ of
% % Section~\ref{sec:structure-specific}.
% of the structure-specific regret section.
% Permuting the node labels before
% computing $\widehat E_k$ (the shuffle control) preserves the projected
% dimension but destroys the alignment, collapsing the gain.
\textsc{GraphDR-LinUCB}.  The graph determines a bottom-$k$
  low-frequency Laplacian eigenbasis $\widehat E_k$; arm features are projected
  to $\R^k$ and LinUCB runs there.  Only the high-frequency residual
  $r_{\perp,k}$ left outside the subspace can hurt the learner; its effect
  is governed by the realized graph quantities $\SRL_T$ and $\CRW_T$.
  % (Section~\ref{sec:structure-specific}).
  The shuffle control permutes node labels before computing $\widehat E_k$,
  preserving the projected dimension but destroying graph alignment and
  eliminating the regret advantage.
}
\label{fig:overview}
\end{figure*}
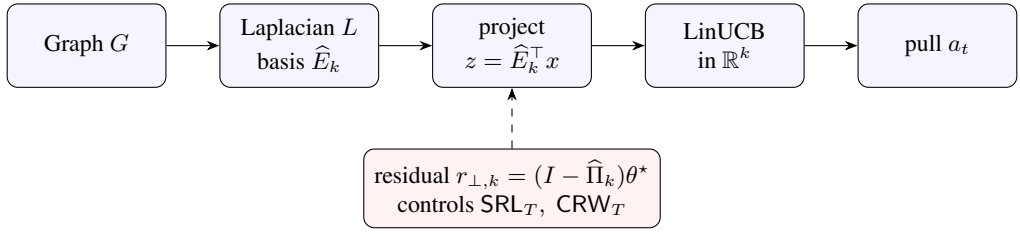

% ─────────────────────────────────────────────────────────────────────────────
\section{Theory}\label{sec:theory}

We present three theorems in order of increasing sharpness.
Table~\ref{tab:theory-summary} summarizes all three regimes; proofs are in the appendix section.
% ~\ref{app:proofs}.

\begin{table}[t]
\centering\small
\begin{tabular}{p{0.20\linewidth}p{0.20\linewidth}p{0.50\linewidth}}
\toprule
Regime & Assumption & Regret guarantee \\
\midrule
Exact subspace (Thm.~\ref{thm:exact}) &
$\theta^\star\in\range(\Pi_k)$, known $E_k$ &
$\wtO(k\sqrt T)$, dimension $d\!\to\!k$ \\[2pt]
Generic robust (Thm.~\ref{thm:generic-main}) &
tail $\norm{r_k}\le\zeta_k$, $\norm{\widehat L-L}\le\varepsilon_L$ &
$\wtO\!\big(k\sqrt T+L_x(\zeta_k+S_k\varepsilon_L/\Delta_k)\,kT\big)$
(linear bias) \\[2pt]
Structure-specific oracle (Thm.~\ref{thm:structured}) &
realized trajectory &
$\;\le 2\beta_{T+1}^0A_T+2\SRL_T A_T+\CRW_T$ \\[2pt]
Estimated residual (Thm.~\ref{thm:estimated-residual}) &
pilot $\widehat r$, $\norm{\widehat r-r_{\perp,k}}\le e_r$ &
oracle bound $+\,\wtO(e_r kT)$; implementable \\
\bottomrule
\end{tabular}
\caption{
Summary of the four theoretical regimes. The structure-specific bound (Thm.~\ref{thm:structured}) tightens the generic worst-case linear bias to realized graph quantities $\SRL_T$ and $\CRW_T$ (Corollary~\ref{cor:benign}); the estimated-residual variant (Thm.~\ref{thm:estimated-residual}) makes it implementable.
% Summary of the three theoretical regimes.  The generic robust bound
% charges a worst-case linear bias; the structure-specific theorem replaces it
% with the realized graph quantities $\SRL_T$ and $\CRW_T$, which are sublinear
% on benign residuals (Corollary~\ref{cor:benign}).  The estimated-residual
% theorem makes the oracle bound implementable at the price of a
% pilot-dependent penalty.
}
\label{tab:theory-summary}
\end{table}

\subsection{Exact known-subspace bound}\label{sec:exact}

When the reward is perfectly smooth on the graph, projection onto the
Laplacian eigenspace is lossless and the problem reduces exactly to a
$k$-dimensional linear bandit.

\begin{assumption}[Exact graph smoothness]\label{ass:exact-smooth}
There exists $\alpha^\star\in\R^k$ such that $\theta^\star=E_k\alpha^\star$.
Equivalently, $\theta^\star\in\range(\Pi_k)$.
\end{assumption}

Under Assumption~\ref{ass:exact-smooth}, $z_{t,a}=E_k^\top x_{t,a}$ satisfies
\[
x_{t,a}^\top\theta^\star
=x_{t,a}^\top E_k\alpha^\star
=(E_k^\top x_{t,a})^\top\alpha^\star
=z_{t,a}^\top\alpha^\star,
\]
so the projected problem is an exact $k$-dimensional linear bandit and the
standard LinUCB analysis applies verbatim in $\R^k$.

\begin{assumption}[Bandit regularity]\label{ass:bandit}
For all $t$ and $a\in\A_t$, $\norm{z_{t,a}}_2\le L_x$.  The projected
parameter satisfies $\norm{\alpha^\star}_2\le S$.  The noise sequence is
conditionally $R$-sub-Gaussian with respect to the filtration $\F_t$ generated
by the history before observing $y_t$: $\E[\exp(s\eta_t)\mid\F_t]\le
\exp(s^2R^2/2)$ for every $s\in\R$.  The arm sets and features may be chosen
adaptively from the past but are fixed before $\eta_t$ is realized.
\end{assumption}

Define pseudo-regret

$\Reg(T)=\sum_{t=1}^T\big(\max_{a\in\A_t}x_{t,a}^\top\theta^\star
-x_{t,a_t}^\top\theta^\star\big)$.

\begin{theorem}[Reduced-dimension regret under exact smoothness]\label{thm:exact}
Under Assumptions~\ref{ass:exact-smooth} and~\ref{ass:bandit}, suppose the
algorithm uses the true projection $E_k^\top$, regularization
$\lambda\ge L_x^2$, and confidence radius
\[
\beta_t
=R\sqrt{k\log\!\left(1+\tfrac{(t-1)L_x^2}{\lambda k}\right)+2\log\tfrac{1}{\delta}}
+\sqrt\lambda S.
\]
Then, with probability at least $1-\delta$, simultaneously for all $T\ge1$,
\[
\Reg(T)\le2\beta_{T+1}\sqrt{2Tk\log\!\left(1+\tfrac{TL_x^2}{\lambda k}\right)}.
\]
Consequently, for fixed $R,L_x,S,\lambda,\delta$, $\Reg(T)=\wtO(k\sqrt T)$.
\end{theorem}

\begin{corollary}[Comparison with full-dimensional LinUCB]\label{cor:full}
Full-dimensional LinUCB in $\R^d$ has rate $\wtO(d\sqrt T)$.  Under exact
graph smoothness and $k\ll d$, \textsc{GraphDR-LinUCB} reduces the dimension
dependence from $d$ to $k$.
\end{corollary}

\subsection{Robust gap-dependent bound}\label{sec:robust}

In practice, two idealizations of Theorem~\ref{thm:exact} fail simultaneously:
the reward may carry energy outside the first $k$ graph frequencies
($\zeta_k>0$), and the eigenspace is estimated from a noisy Laplacian
($\varepsilon_L>0$).  The following handles both.  The key tool is a
projector-perturbation bound from Davis and Kahan.

\begin{lemma}[Davis--Kahan projector perturbation \cite{davis1970rotation}]\label{lem:dk}
Assume $L$ and $\widehat L$ are symmetric and
$\norm{\widehat L-L}_2\le\varepsilon_L\le\Delta_k/4$.  Then
$\norm{\widehat\Pi_k-\Pi_k}_2\le4\varepsilon_L/\Delta_k$, and consequently
$\norm{(I-\widehat\Pi_k)\Pi_k}_2\le4\varepsilon_L/\Delta_k$.
\end{lemma}

\begin{assumption}[Approximate smoothness and noisy graph]\label{ass:robust}
Let $r_k=(I-\Pi_k)\theta^\star$.  Assume $\norm{\Pi_k\theta^\star}_2\le S_k$
and $\norm{r_k}_2\le\zeta_k$.  The ambient features satisfy
$\norm{x_{t,a}}_2\le L_x$ for all $t,a$.  The noise is conditionally
$R$-sub-Gaussian as in Assumption~\ref{ass:bandit}, and the arm sets are
chosen before the current noise is realized.
\end{assumption}

\begin{theorem}[Generic robust regret bound]\label{thm:generic-main}
Assume Lemma~\ref{lem:dk} and Assumption~\ref{ass:robust}.  Run
\textsc{GraphDR-LinUCB} with $\widehat P=\widehat E_k^\top$ and
$\lambda\ge L_x^2$, with confidence radius
$\beta_t=\beta_t^0+\nu_k\sqrt{2(t-1)k\log(1+(t-1)L_x^2/(\lambda k))}$
(equivalently $\beta_t^0+\nu_kA_T$ for known horizon $T$), which requires
$\nu_k$ known or upper-bounded.  Define
% \[
% \gamma_k=\frac{4\varepsilon_L}{\Delta_k},
% \qquad
% \nu_k=L_x(\zeta_k+S_k\gamma_k)=L_x\!\left(\zeta_k+\tfrac{4S_k\varepsilon_L}{\Delta_k}\right),
% \qquad
% \overline S_k=S_k+\zeta_k,
% \]

\begin{align*}
  \gamma_k &= \frac{4\varepsilon_L}{\Delta_k}, \\
  \nu_k    &= L_x(\zeta_k+S_k\gamma_k)
              = L_x\!\left(\zeta_k+\tfrac{4S_k\varepsilon_L}{\Delta_k}\right), \\
  \overline S_k &= S_k+\zeta_k.
\end{align*}
\begin{gather*}
  \beta_t^0 = R\sqrt{k\log\!\left(1+\tfrac{(t-1)L_x^2}{\lambda k}\right)+2\log\tfrac1\delta}
              +\sqrt\lambda\,\overline S_k, \\[4pt]
  A_T = \sqrt{2Tk\log\!\left(1+\tfrac{TL_x^2}{\lambda k}\right)}.
\end{gather*}

Then, with probability at least $1-\delta$, simultaneously for all $T\ge1$,

% \[
% \Reg(T)\le2\beta_{T+1}^0A_T+2\nu_kA_T^2+2\nu_kT
% =2\beta_{T+1}^0A_T+4\nu_kTk\log\!\left(1+\tfrac{TL_x^2}{\lambda k}\right)+2\nu_kT.
% \]

 \begin{multline*}
  \Reg(T)\le2\beta_{T+1}^0A_T+2\nu_kA_T^2+2\nu_kT \\
  =2\beta_{T+1}^0A_T
    +4\nu_kTk\log\!\left(1+\tfrac{TL_x^2}{\lambda k}\right)+2\nu_kT.
  \end{multline*}
Suppressing logarithmic factors,
$\Reg(T)=\wtO\!\big(k\sqrt T+L_x(\zeta_k+S_k\varepsilon_L/\Delta_k)kT\big)$.
If $\zeta_k=0$ and $\varepsilon_L=0$, then $\nu_k=0$ and the bound reduces to
$\wtO(k\sqrt T)$.  If $\nu_k$ is unknown it must be replaced by an upper
bound or selected by a doubling or grid wrapper; otherwise the result is an
oracle bound.
\end{theorem}

The $\varepsilon_L/\Delta_k$ dependence is the key takeaway: a large spectral
gap protects against graph noise, while a small gap amplifies it.

\subsection{Main result: structure-specific residual regret}\label{sec:structure-specific}

The generic bound of Theorem~\ref{thm:generic-main} uses only the scalar
envelope $|\xi_{t,a}|\le\nu_k$, treating the high-frequency residual
$(I-\widehat\Pi_k)\theta^\star$ as an arbitrary adversary.  In graph problems
this residual is not arbitrary: it is a high-frequency graph signal, and it
harms the learner only through the arms that actually appear and the
low-frequency coordinates that are actually played.  The following theorem
retains this structure instead of collapsing it into $\nu_k$.

Let $r_{\perp,k}=(I-\widehat\Pi_k)\theta^\star$ and
$\xi_t(a)=x_{t,a}^\top r_{\perp,k}$.

\begin{definition}[Residual leverage and candidate residual width]\label{def:srl}
For a realized trajectory, define the \emph{residual leverage}
\[
\SRL_T(r_{\perp,k})
=\max_{1\le t\le T+1}\norm{\textstyle\sum_{s=1}^{t-1}z_{s,a_s}\,\xi_s(a_s)}_{V_t^{-1}},
\]
and the \emph{cumulative candidate residual width}
\[
\CRW_T(r_{\perp,k})
=\sum_{t=1}^T\Big(\max_{a\in\A_t}\xi_t(a)-\min_{a\in\A_t}\xi_t(a)\Big).
\]
\end{definition}

$\SRL_T$ is the self-normalized bias that the high-frequency residual injects
into the regression updates; $\CRW_T$ is the amount by which the residual can
change the ordering of arms within each candidate set.  In the node-arm case
($x_{t,a}=e_a$, $r_{\perp,k}=U_{>k}c_{>k}$), both are determined by where
the high-frequency graph signal lives and which arms the bandit actually
compares: a strictly finer characterization than the scalar tail $\zeta_k$.

\begin{theorem}[A posteriori oracle residual bound]\label{thm:structured}
Assume the noise and bounded projected-feature conditions in
Assumption~\ref{ass:bandit}.  Fix a horizon $T$ and suppose
\textsc{GraphDR-LinUCB} is run with $\widehat E_k$, $\lambda\ge L_x^2$, and
confidence radius $\beta_t=\beta_t^0+B_T$ with
$B_T\ge\SRL_T(r_{\perp,k})$, where
% \[
% \beta_t^0
% =R\sqrt{k\log\!\left(1+\tfrac{(t-1)L_x^2}{\lambda k}\right)+2\log\tfrac1\delta}
% +\sqrt\lambda\,\norm{\widehat E_k^\top\theta^\star}_2,
% \qquad
% A_T=\sqrt{2Tk\log\!\left(1+\tfrac{TL_x^2}{\lambda k}\right)}.
% \]
\begin{gather*}
  \beta_t^0 = R\sqrt{k\log\!\left(1+\tfrac{(t-1)L_x^2}{\lambda k}\right)+2\log\tfrac1\delta}
              +\sqrt\lambda\,\norm{\widehat E_k^\top\theta^\star}_2, \\[4pt]
  A_T = \sqrt{2Tk\log\!\left(1+\tfrac{TL_x^2}{\lambda k}\right)}.
  \end{gather*}
Then, with probability at least $1-\delta$,
\[
\Reg(T)\le2\beta_{T+1}^0A_T+2B_TA_T+\CRW_T(r_{\perp,k}).
\]
In particular, the oracle choice $B_T=\SRL_T(r_{\perp,k})$ yields
$\Reg(T)\le2\beta_{T+1}^0A_T+2\SRL_T(r_{\perp,k})A_T+\CRW_T(r_{\perp,k})$.
\end{theorem}

This is an a posteriori oracle inequality: $B_T$ depends on the unknown
$\theta^\star$ and the realized trajectory.  Proposition~\ref{prop:no-inflation}
below characterizes exactly what is lost when the confidence radius is not
inflated, explaining why an estimate of $\SRL_T$ is needed for an
implementable algorithm.

\begin{proposition}[What happens without residual-radius inflation]\label{prop:no-inflation}
If the same algorithm is run with the smaller radius $\beta_t^0$ only, the
same proof gives the valid bound
% \[
% \Reg(T)\le2\beta_{T+1}^0A_T+\SRL_T(r_{\perp,k})A_T+T\,\SRL_T(r_{\perp,k})
% +\CRW_T(r_{\perp,k}).
% \]
\begin{multline*}
  \Reg(T)\le2\beta_{T+1}^0A_T+\SRL_T(r_{\perp,k})A_T \\
  +T\,\SRL_T(r_{\perp,k})+\CRW_T(r_{\perp,k}).
  \end{multline*}
Thus the clean $2\SRL_TA_T$ oracle term requires either a known upper bound
on $\SRL_T$ or an estimated-residual confidence inflation.  Without inflation,
the comparator arm is not the played arm, and its uncertainty cannot be
collapsed by the elliptical-potential lemma.
\end{proposition}

The oracle bound becomes implementable by replacing the unknown residual with
a pilot estimate.

\begin{theorem}[Estimated-residual horizon oracle bound]\label{thm:estimated-residual}
Let $\widehat r$ be a pilot estimate of $r_{\perp,k}$ obtained from data
independent of the bandit noise used in the regret run, with
$\norm{\widehat r-r_{\perp,k}}_2\le e_r$ and $\norm{x_{t,a}}_2\le L_x$.
Define $\widehat\xi_t(a)=x_{t,a}^\top\widehat r$ and a computable leverage
bound
$\widehat B_T\ge\max_{1\le t\le T+1}\norm{\sum_{s<t}z_{s,a_s}\widehat\xi_s(a_s)}_{V_t^{-1}}$.
Run \textsc{GraphDR-LinUCB} with radius
$\beta_t=\beta_t^0+\widehat B_T+L_xe_rA_T$.  Then, on the same
high-probability event as Theorem~\ref{thm:structured},
% \[
% \Reg(T)\le2\beta_{T+1}^0A_T+2(\widehat B_T+L_xe_rA_T)A_T
% +\widehat{\CRW}_T(\widehat r)+2L_xe_rT,
% \]
 \begin{multline*}
  \Reg(T)\le2\beta_{T+1}^0A_T+2(\widehat B_T+L_xe_rA_T)A_T \\
  +\widehat{\CRW}_T(\widehat r)+2L_xe_rT,
  \end{multline*}
where $\widehat{\CRW}_T(\widehat r)=\sum_{t=1}^T(\max_a x_{t,a}^\top\widehat
r -\min_a x_{t,a}^\top\widehat r)$.  To be genuinely online, use the
predictable radii $\beta_t=\beta_t^0+\widehat B_t+L_xe_rA_t$ with
$\widehat B_t=\max_{u\le t}\norm{\sum_{s<u}z_s\widehat\xi_s(a_s)}_{V_u^{-1}}$;
the stated horizon-$T$ form is the special case at $t=T$.
\end{theorem}

Two corollaries bound the range of behaviour.  First, the generic theorem is
only a loose special case of Theorem~\ref{thm:structured}.

\begin{corollary}[Generic theorem as a loose special case]\label{cor:generic-from-structured}
If $|\xi_t(a)|\le\nu_k$ for all $t,a$, then $\SRL_T(r_{\perp,k})\le\nu_kA_T$
and $\CRW_T(r_{\perp,k})\le2\nu_kT$.  Substituting these into the oracle
version of Theorem~\ref{thm:structured} with $B_T=\nu_kA_T$ recovers
Theorem~\ref{thm:generic-main}, up to the harmless replacement of
$\norm{\widehat E_k^\top\theta^\star}_2$ by $\overline S_k$.
\end{corollary}

\begin{corollary}[Benign residuals yield sublinear extra bias]\label{cor:benign}
On any problem sequence with $\SRL_T(r_{\perp,k})=O(\sqrt{k\log T})$ and
$\CRW_T(r_{\perp,k})=O(\sqrt T)$, the oracle-inflated algorithm of
Theorem~\ref{thm:structured} gives
$\Reg(T)=\wtO(k\sqrt T)+\wtO(k\sqrt T)+O(\sqrt T)=\wtO(k\sqrt T)$, since
$\SRL_TA_T=\wtO(k\sqrt T)$, so the high-frequency residual contributes only
sublinear regret.  These conditions hold, for example, when the residual is
nearly common-mode inside most candidate sets and its product with the played
low-frequency features exhibits self-normalized cancellation rather than
worst-case alignment.
\end{corollary}

% ─────────────────────────────────────────────────────────────────────────────
\section{Synthetic Experiments}\label{sec:exp}

We validate the dimension-reduction mechanism on controlled synthetic graphs
where the reward subspace is known by construction.  Arms are graph nodes with
one-hot features ($d=n$), and $\theta^\star=U_k\alpha^\star/\norm{U_k\alpha^\star}$
uses the true bottom-$k$ nontrivial Laplacian eigenvectors, so the reward is
graph-smooth by design.  All methods compared are: Random; LinUCB-full
($d=n$); principal component analysis (PCA)+LinUCB; Johnson-Lindenstrauss (JL)+LinUCB; \textsc{GraphDR}+LinUCB; and
\textbf{GraphDR-shuffled}, the graph-shuffle control that projects with the
bottom-$k$ nontrivial eigenvectors of a node-permuted copy of the same graph.
The shuffle control has the same projected dimension as GraphDR but a
graph-misaligned eigenspace; any genuine gain must collapse when it is applied, and in every experiment below, it does.  Regret $R(T)$ is the mean $\pm$
standard error over $8$ seeds; LinUCB uses Sherman-Morrison updates.

\paragraph{Main comparison.}
On a stochastic block model (SBM) with $n=d=200$, $C=k^\star=5$ communities,
and $T=20{,}000$, \textsc{GraphDR-LinUCB} reduces regret by more than $15\times$
relative to full LinUCB and $69\times$ relative to PCA
(Figure~\ref{fig:main}):
\begin{center}\small
\begin{tabular}{lcc}
\toprule
Method & $R(T)\downarrow$ & $\pm$ s.e.m. \\
\midrule
PCA+LinUCB & $2152.3$ & $9.9$ \\
Random & $1226.0$ & $2.8$ \\
GraphDR-shuffled (control) & $957.8$ & $6.9$ \\
JL+LinUCB & $777.8$ & $4.5$ \\
LinUCB-full ($d=200$) & $480.8$ & $2.7$ \\
GraphDR+LinUCB & $\mathbf{31.3}$ & $0.7$ \\
\bottomrule
\end{tabular}
\end{center}

\begin{figure}[t]
\centering
\includegraphics[width=0.8\linewidth]{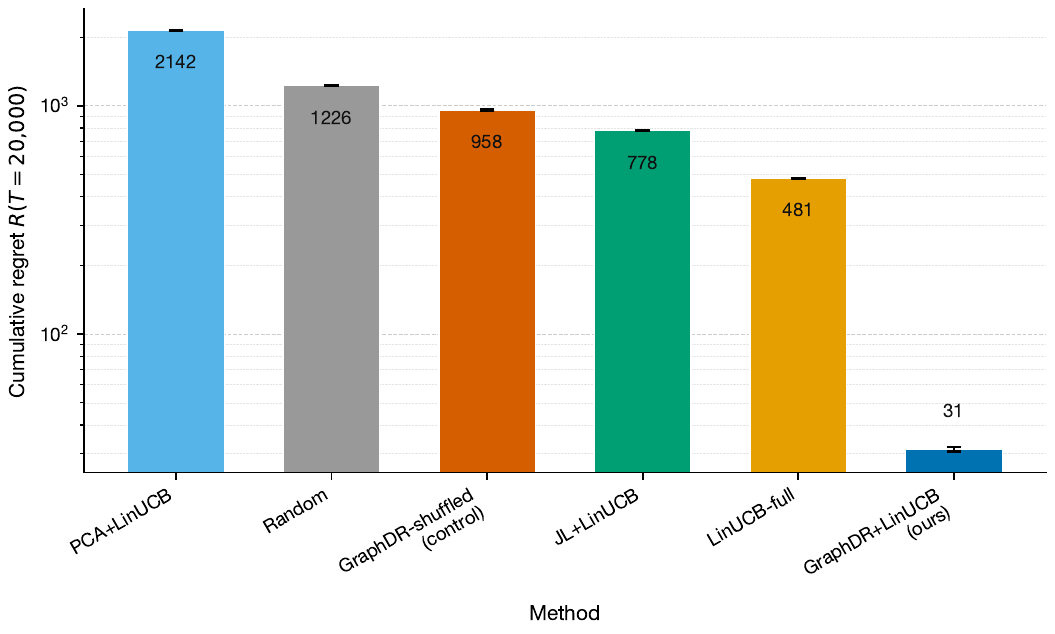}
\caption{Main synthetic comparison with the graph-shuffle control.
\textsc{GraphDR}+LinUCB dominates; the shuffled eigenspace collapses toward
the random-policy level.}
\label{fig:main}
\end{figure}

The shuffle control collapses toward the random-policy level ($957.8$), showing
the gain is structural rather than merely dimensional.  PCA performs poorly
because one-hot features are isotropic and carry no variance signal.

\paragraph{Dimension scaling.}
Fixing $k^\star=5$ and scaling $d=n$ from $60$ to $800$, GraphDR regret stays
nearly flat (from $22.2$ to $40.6$, a $1.8\times$ growth) while full-dimensional
LinUCB grows steeply (from $183.0$ to $852.1$, a $4.7\times$ growth),
matching the predicted $k$-versus-$d$ exploration cost of
Corollary~\ref{cor:full} (Figure~\ref{fig:dim}). 
% full table in Appendix~\ref{app:synthetic}).

\begin{figure}[t]
\centering
\includegraphics[width=0.8\linewidth]{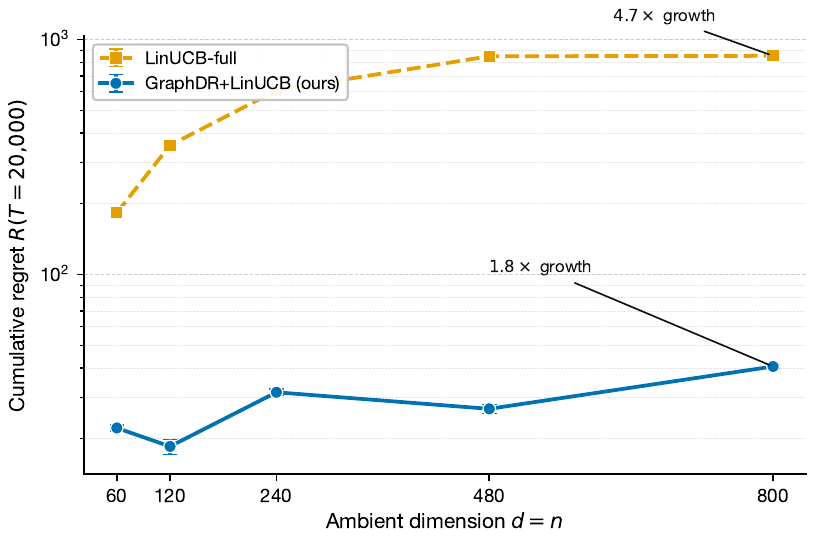}
\caption{Dimension scaling at fixed $k^\star=5$.  GraphDR-LinUCB remains
nearly flat as $d=n$ grows; full LinUCB grows with ambient dimension.}
\label{fig:dim}
\end{figure}

\paragraph{Second graph family.}
On a random geometric graph (RGG, $n=200$, radius $0.16$) the same pattern
holds: GraphDR attains $R(T)=18.1$ against $417.8$ for full LinUCB and over
$2000$ for PCA, JL, and Random, with the shuffle control collapsing to
$2544.6$ (Figure~\ref{fig:rgg}).

\begin{figure}[t]
\centering
\includegraphics[width=0.8\linewidth]{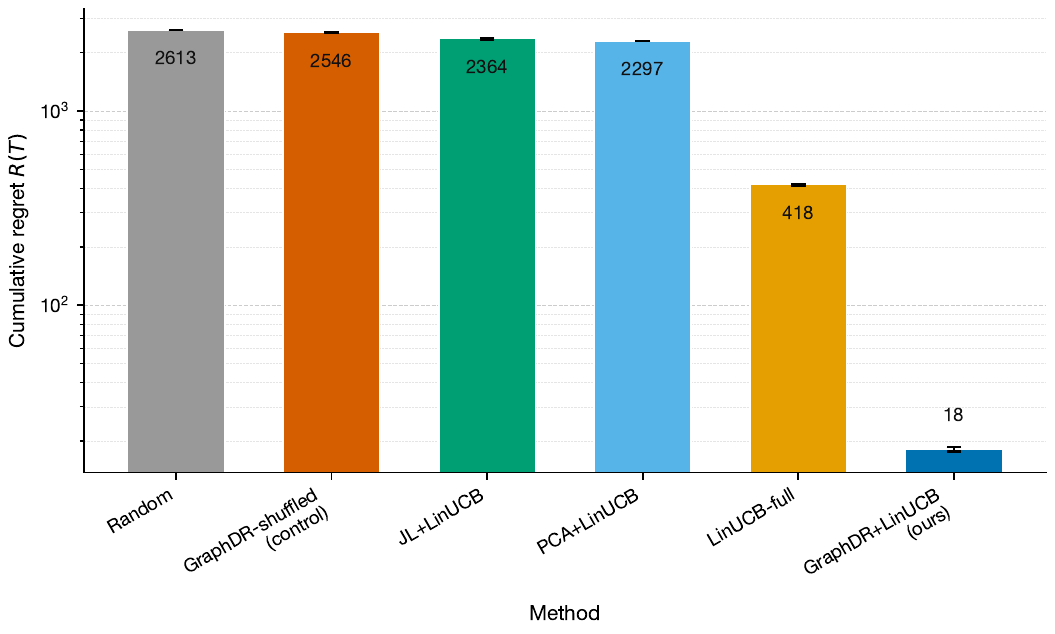}
\caption{Second graph family: random geometric graph.  GraphDR dominates; the
graph-shuffle control collapses.}
\label{fig:rgg}
\end{figure}

\paragraph{$k$-sweep.}
Sweeping the projected dimension at fixed $k^\star=5$ on the SBM confirms
that regret is minimized exactly at $k=5$ and rises on both sides: too few
dimensions omit signal; too many reintroduce exploration cost.  The regret at
$k=1,2,3,\mathbf{5},8,12,20$ is
$419.9,384.2,98.6,\mathbf{28.7},39.5,55.4,69.5$ (Figure~\ref{fig:ksweep}).

\begin{figure}[t]
\centering
\includegraphics[width=0.8\linewidth]{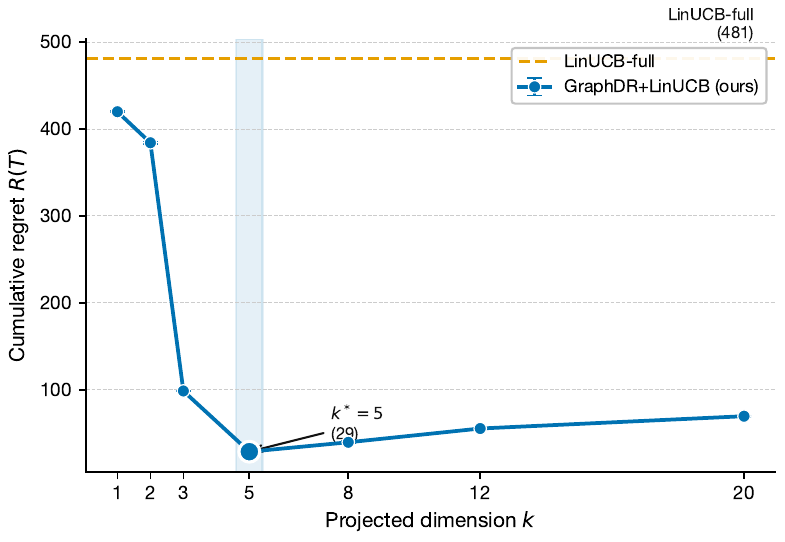}
\caption{Regret versus projected dimension $k$.  The minimum is at the true
spectral dimension $k^\star=5$.}
\label{fig:ksweep}
\end{figure}

\paragraph{Noise stress test and robust theorem validation.}
When $G$ is observed under edge-flip noise, $k$ must be selected from a noisy
spectrum.  The bootstrap and naive eigengap estimators return identical
$\widehat k$ at every noise level; under heavier noise $\widehat k$ degrades
($4\!\to\!2\!\to\!1$) with measurable regret cost.  A complementary stress
test directly probes the misspecification envelope $\nu_k$ by measuring
$\zeta_k$ and $\varepsilon_L$ in each run: regret rises monotonically with
both terms (pooled Spearman $0.68$ across $16$ runs).
% (full tables in Appendix~\ref{app:synthetic}).

% ─────────────────────────────────────────────────────────────────────────────

\begin{table*}[t]
\centering\small
\begin{tabular}{lccccc}
\toprule
Dataset & $\zeta_k$ & GraphDR & PCA & full LinUCB & shuffle \\
\midrule
MovieLens-100k & $0.67$ & $\mathbf{238.4}$ & $331.5$ & $615.0$ & $712.7$ \\
MovieLens-1M   & $0.68$ & $273.3$ & $\mathbf{186.7}$ & $592.6$ & $650.2$ \\
Amazon (Dig.\ Music) & $0.86$ & $\mathbf{309.9}$ & $338.6$ & $463.4$ & $526.4$ \\
LastFM (HetRec) & $0.90$ & $\mathbf{412.1}$ & $508.5$ & $678.9$ & $774.3$ \\
ogbn-arxiv & $0.96$ & $667.6$ & $\mathbf{258.5}$ & $492.2$ & $602.1$ \\
MIND-small & $0.99$ & $\mathbf{415.0}$ & $887.7$ & $780.5$ & $943.9$ \\
\bottomrule
\end{tabular}
\caption{Cumulative regret $R (T=20{,}000)$ across six real graph-bandit
datasets (top-1500 nodes each, $5$ seeds; lowest in \textbf{bold}).
\textsc{GraphDR-LinUCB} beats full LinUCB on five of six datasets, failing
only on ogbn-arxiv where the citation-graph eigenspace is misaligned with the
subject-label reward.  The graph-shuffle control is always near-random.
GraphDR wins on $4/6$ against PCA; the scalar tail $\zeta_k$ partially but
not fully predicts the pattern.}
\label{tab:crossdata}
\end{table*}

\begin{table*}[t]
\centering\small
\begin{tabular}{lccccc}
\toprule
Dataset & GraphDR & SpectralUCB & Laplacian-reg & PCA & full LinUCB \\
\midrule
Synthetic SBM & $\mathbf{31.3}$ & $466.2$ & $462.1$ & $2157.6$ & $484.9$ \\
MovieLens-100k & $\mathbf{122.7}$ & $256.0$ & $247.2$ & $184.7$ & $259.1$ \\
MovieLens-1M & $\mathbf{147.7}$ & $245.1$ & $238.0$ & $199.4$ & $250.6$ \\
Amazon & $\mathbf{148.6}$ & $195.0$ & $189.4$ & $227.2$ & $196.0$ \\
LastFM & $\mathbf{186.4}$ & $282.7$ & $276.9$ & $286.3$ & $285.9$ \\
MIND-small & $\mathbf{169.4}$ & $341.7$ & $325.9$ & $358.1$ & $346.1$ \\
ogbn-arxiv & $277.0$ & $210.8$ & $\mathbf{203.2}$ & $228.2$ & $214.8$ \\
\bottomrule
\end{tabular}
\caption{Cumulative regret $R(T)$ against graph-aware bandits (
% $T{=}8000$, $3$ seeds real; 
$T{=}20000$, $5$ seeds synthetic; lowest per row in \textbf{bold}).  \textsc{GraphDR-LinUCB} beats both SpectralUCB and
Laplacian-regularized LinUCB on the synthetic benchmark and on five of six
real datasets.  The single exception is ogbn-arxiv ($\zeta_k{=}0.96$): even
soft graph-aware methods beat hard bottom-$k$ projection when the eigenspace
is misaligned with the reward.  The advantage of hard spectral truncation over
soft eigenvalue-penalized smoothing is that truncation entirely removes the
exploration cost of high-frequency directions rather than merely down-weighting
it.}
\label{tab:e3}
\end{table*}

\section{Real-Data Benchmarks}\label{sec:realdata}

We evaluate on six real graph-bandit datasets: four recommendation graphs
(MovieLens-100k/1M \cite{li2010contextual}, Amazon Digital Music, LastFM/HetRec) and two
non-recommendation graphs (MIND-small news \cite{dudik2011doubly}, ogbn-arxiv citation \cite{hu2020ogb}).  Items,
artists, papers, or news articles are arms; features are node indicators; the
graph is built from co-rating, co-listening, citation, or category and
co-click relations; and rewards are calibrated from real ratings, plays,
labels, or click-through statistics.  All runs report cumulative regret at
$T=20{,}000$ over $5$ seeds.  Full dataset construction, per-dataset
MovieLens results, and the measured spectral quantities ($\widehat k$,
eigengap, and $\zeta_k$) are in the appendix.
% ~\ref{app:realdata}.

\paragraph{Cross-dataset summary.}
Table~\ref{tab:crossdata} shows \textsc{GraphDR-LinUCB} beats full LinUCB on
five of six datasets; the exception is ogbn-arxiv, where the citation-graph
eigenspace is misaligned with the subject-label reward ($\zeta_k=0.96$).
The GraphDR-versus-PCA comparison is more nuanced:
GraphDR leads on four datasets (MovieLens-100k, Amazon, LastFM,
MIND), while PCA wins on MovieLens-1M and ogbn-arxiv.  At scale and when the
reward is only weakly smooth, learned content variance can carry signal that
PCA exploits better than the graph.  The scalar reward tail $\zeta_k$
partially but not fully predicts this: ogbn-arxiv ($0.96$) is GraphDR's
clearest failure, yet MIND has the largest tail ($0.99$) and GraphDR wins
decisively.  The structure-specific theorem explains the discrepancy; what
matters is not total high-frequency energy but whether the residual creates
large candidate-set width or residual leverage along the played path.
% (Appendix~\ref{app:diagnostic-main}).

\paragraph{Comparison against graph-aware baselines.}
The decisive test for a graph-DR contribution is whether the algorithm beats
bandits that already exploit the graph.  Table~\ref{tab:e3} compares against
SpectralUCB \cite{valko2014spectral}: LinUCB in the full Laplacian
eigenbasis with a per-eigenvalue smoothness penalty and no hard truncation, 
and Laplacian-regularized LinUCB \cite{yang2020laplacian}: full-$d$ LinUCB
with graph-smoothness regularization replacing ridge.

\textsc{GraphDR-LinUCB}'s advantage holds against graph-aware competitors on
every graph-smooth dataset; on the misaligned ogbn-arxiv ($\zeta_k=0.96$),
Laplacian-regularized LinUCB is the safer choice.

To isolate the projection map from the feature source, a matched experiment replaces content-PCA with graph-only PCA (top-$k$ singular vectors of $A$, the same graph-only input GraphDR receives): both graph-spectral maps beat graph-blind baselines on most datasets, and GraphDR and graph-PCA split $3$--$3$ given identical input, confirming the advantage is structural.

\paragraph{Diagnostic.}
A theory-motivated diagnostic $\mathcal{G}_k$ based on the residual quantities of Theorem~\ref{thm:structured} was evaluated but failed to outperform $\zeta_k$ on the signed gap; we replace it with the computable $\Gamma_k$ below.

\subsection{A working selection rule: the subspace-capture margin}\label{sec:capture-rule}
The structure-specific theorem says a reducer hurts the learner only through the reward
energy it fails to keep. The GraphDR-vs-PCA contest is therefore most directly predicted
not by the absolute graph tail $\zeta_k$, but by a \emph{comparison} of how much
reward energy each reducer's subspace captures. Let $U_k$ be GraphDR's bottom-$k$
nontrivial Laplacian eigenbasis and $P$ the competing reducer's $k$-dimensional basis
(content-PCA here). Define the \textbf{subspace-capture margin}
\begin{align*}
\Gamma_k\;&=\;\norm{U_k^\top\theta}_2^2\;-\;\norm{P^\top\theta}_2^2 \\
&=\;\underbrace{\norm{U_k^\top\theta}_2^2}_{\text{graph low-freq energy kept}}
\;-\;\underbrace{\norm{P^\top\theta}_2^2}_{\text{competitor energy kept}} ,
\end{align*}
the (squared, unit-norm-reward) reward energy GraphDR's subspace captures minus that the
competitor captures. $\Gamma_k>0$ predicts GraphDR; $\Gamma_k<0$ predicts the competitor.
Unlike $\zeta_k$, which only sees GraphDR's own tail, $\Gamma_k$ is a head-to-head
alignment statistic and is computable from the same pilot reward estimate $\widehat\theta$
the diagnostic $\mathcal{G}_k$ already requires.

Validated on all six real datasets, the fitted-threshold rule is leave-one-out correct on $6/6$; the zero-parameter sign rule $\Gamma_k>0$ is correct on $5/6$, with the lone miss at the near-tie $\Gamma_k=-0.008$ (MIND-small).

Against the dead diagnostic $\mathcal{G}_k$ ($|\rho_s|\approx0.11$, $2/6$ leave-one-out) and the scalar tail $\zeta_k$ ($2/6$), $\Gamma_k$ is the clear headline rule; the eigengap $\Delta_k$ correlates at $0.83$ but reaches only $4/6$ leave-one-out and lacks mechanistic justification for the head-to-head comparison.

% ─────────────────────────────────────────────────────────────────────────────
\section{Conclusion}\label{sec:conclusion}

We presented \textsc{GraphDR-LinUCB}, a reduction from graph-smooth contextual
  bandits to $k$-dimensional linear bandits via Laplacian eigenspace projection.
  The central theoretical insight is that the high-frequency reward component need
  not cost a linear-in-$T$ penalty: its actual cost is governed by two realized graph
  quantities, residual leverage and candidate-set residual width, rather than by the
  tail energy $\zeta_k$ alone. This structure-specific view replaces worst-case
  misspecification theory with a finer, graph-aware accounting, and the subspace-capture
  margin $\Gamma_k$ translates this accounting into a practical, threshold-free
  selection rule. Graph-shuffle controls on every configuration confirm the gain is
  structural, and synthetic and real-data experiments validate the approach across
  recommendation, news, and citation graphs. 
\bibliography{aaai2027}

\appendix

% ─────────────────────────────────────────────────────────────────────────────
\section{Proofs}\label{app:proofs}

This appendix contains the proofs of all results stated in
% Section~\ref{sec:theory}, in the order they appear. 
the Theory section, in the order they appear.
Each proof begins with a
brief sketch of the key step to orient before the formal argument.

\begin{proof}[Proof of Theorem~\ref{thm:exact}]
\emph{Sketch.} The projection $E_k^\top$ makes the problem exactly
$k$-dimensional; the standard self-normalized concentration and elliptical
potential arguments then apply directly in $\R^k$.

Let $z_t=z_{t,a_t}$.  Under Assumption~\ref{ass:exact-smooth},
$y_t=z_t^\top\alpha^\star+\eta_t$.  The ridge estimate satisfies
\[
\widehat\alpha_t-\alpha^\star
=V_t^{-1}\sum_{s=1}^{t-1}z_s\eta_s-\lambda V_t^{-1}\alpha^\star.
\]
Taking the $V_t$-norm and using $V_t\succeq\lambda I_k$,
\[
\norm{\widehat\alpha_t-\alpha^\star}_{V_t}
\le\norm{\textstyle\sum_{s=1}^{t-1}z_s\eta_s}_{V_t^{-1}}+\sqrt\lambda S.
\]
The self-normalized martingale inequality gives, with probability at least
$1-\delta$, uniformly over $t$,
\[
\norm{\textstyle\sum_{s=1}^{t-1}z_s\eta_s}_{V_t^{-1}}
\le R\sqrt{2\log\!\left(\tfrac{\det(V_t)^{1/2}}{\det(\lambda I_k)^{1/2}\delta}\right)}.
\]
Since $\norm{z_s}_2\le L_x$,
$\log\frac{\det(V_t)}{\det(\lambda I_k)}\le k\log(1+\frac{(t-1)L_x^2}{\lambda k})$,
so $\norm{\widehat\alpha_t-\alpha^\star}_{V_t}\le\beta_t$ for all $t$ on this
event.  For any arm, Cauchy--Schwarz gives
$|z_{t,a}^\top(\widehat\alpha_t-\alpha^\star)|\le\beta_t\norm{z_{t,a}}_{V_t^{-1}}$,
so the UCB is optimistic.  Letting $a_t^\star$ be optimal, the standard
optimism argument yields
$x_{t,a_t^\star}^\top\theta^\star-x_{t,a_t}^\top\theta^\star
\le2\beta_t\norm{z_t}_{V_t^{-1}}$.  Summing and using monotonicity of
$\beta_t$, $\Reg(T)\le2\beta_{T+1}\sum_{t=1}^T\norm{z_t}_{V_t^{-1}}$.
Because $\lambda\ge L_x^2$, $\norm{z_t}_{V_t^{-1}}^2\le1$, and the
elliptical potential lemma gives
$\sum_{t=1}^T\norm{z_t}_{V_t^{-1}}\le\sqrt{2Tk\log(1+\frac{TL_x^2}{\lambda k})}$.
Substitution proves the theorem.
\end{proof}

\begin{proof}[Proof of Lemma~\ref{lem:dk}]
\emph{Sketch.} Weyl's inequality controls the eigenvalue shift; Davis--Kahan
then bounds the projector rotation by the shift divided by the gap.

By Weyl's inequality, each eigenvalue of $\widehat L$ differs from the
corresponding eigenvalue of $L$ by at most $\varepsilon_L$, so the separation
between the perturbed bottom-$k$ cluster and the rest is at least
$\Delta_k-2\varepsilon_L\ge\Delta_k/2$.  The Davis--Kahan $\sin\Theta$ theorem
for symmetric matrices gives
$\norm{\widehat\Pi_k-\Pi_k}_2\le\frac{2\norm{\widehat L-L}_2}{\Delta_k-2\varepsilon_L}
\le4\varepsilon_L/\Delta_k$.  The second inequality holds because
$(I-\widehat\Pi_k)\Pi_k$ is a cross-projector bounded in operator norm by the
distance between projectors.
\end{proof}

\begin{proof}[Proof of Theorem~\ref{thm:generic-main}]
\emph{Sketch.} Decompose $\theta^\star$ into its in-subspace part and
high-frequency residual; bound the residual's effect on the regression updates
by $\nu_k$ pointwise; then apply the standard optimism argument with an
inflated confidence radius.

Decompose $\theta^\star=\widehat\Pi_k\theta^\star+(I-\widehat\Pi_k)\theta^\star$
and set $\widehat\alpha^\star=\widehat E_k^\top\theta^\star$, so
$\widehat\Pi_k\theta^\star=\widehat E_k\widehat\alpha^\star$ and
$\norm{\widehat\alpha^\star}_2\le\norm{\theta^\star}_2\le S_k+\zeta_k=\overline S_k$.
Moreover,
% \[
% \norm{(I-\widehat\Pi_k)\theta^\star}_2
% =\norm{(I-\widehat\Pi_k)(\Pi_k\theta^\star+r_k)}_2
% \le\norm{(I-\widehat\Pi_k)\Pi_k}_2\norm{\Pi_k\theta^\star}_2
% +\norm{(I-\widehat\Pi_k)r_k}_2
% \le\gamma_kS_k+\zeta_k.
% \]
\begin{align*}
  \norm{(I-\widehat\Pi_k)\theta^\star}_2
    &= \norm{(I-\widehat\Pi_k)(\Pi_k\theta^\star+r_k)}_2 \\
    &\le \norm{(I-\widehat\Pi_k)\Pi_k}_2\norm{\Pi_k\theta^\star}_2
       +\norm{(I-\widehat\Pi_k)r_k}_2 \\
    &\le \gamma_kS_k+\zeta_k.
  \end{align*}
For any arm define $\xi_{t,a}=x_{t,a}^\top(I-\widehat\Pi_k)\theta^\star$; then
$|\xi_{t,a}|\le L_x\norm{(I-\widehat\Pi_k)\theta^\star}_2\le
L_x(\zeta_k+S_k\gamma_k)=\nu_k$, and
$x_{t,a}^\top\theta^\star=z_{t,a}^\top\widehat\alpha^\star+\xi_{t,a}$ with
$|\xi_{t,a}|\le\nu_k$.  For the selected arm write $z_t=z_{t,a_t}$,
$\xi_t=\xi_{t,a_t}$, so
$y_t=z_t^\top\widehat\alpha^\star+\eta_t+\xi_t$ and
\[
\widehat\alpha_t-\widehat\alpha^\star
=V_t^{-1}\sum_{s<t}z_s\eta_s+V_t^{-1}\sum_{s<t}z_s\xi_s
-\lambda V_t^{-1}\widehat\alpha^\star.
\]
On the self-normalized event,
$\norm{\sum_{s<t}z_s\eta_s}_{V_t^{-1}}\le R\sqrt{k\log(1+\frac{(t-1)L_x^2}{\lambda
k})+2\log\frac1\delta}$ and
$\lambda\norm{\widehat\alpha^\star}_{V_t^{-1}}\le\sqrt\lambda\,\overline S_k$.
Since $|\xi_s|\le\nu_k$ and $V_t\succeq V_s$,
% \[
% \norm{\textstyle\sum_{s<t}z_s\xi_s}_{V_t^{-1}}
% \le\nu_k\sum_{s<t}\norm{z_s}_{V_s^{-1}}
% \le\nu_k\sqrt{2(t-1)k\log\!\left(1+\tfrac{(t-1)L_x^2}{\lambda k}\right)}.
% \]
\begin{align*}
  \norm{\textstyle\sum_{s<t}z_s\xi_s}_{V_t^{-1}}
    &\le \nu_k\sum_{s<t}\norm{z_s}_{V_s^{-1}} \\
    &\le \nu_k\sqrt{2(t-1)k\log\!\left(1+\tfrac{(t-1)L_x^2}{\lambda k}\right)}.
  \end{align*}
Hence
$\norm{\widehat\alpha_t-\widehat\alpha^\star}_{V_t}\le\beta_t^0
+\nu_k\sqrt{2(t-1)k\log(1+\frac{(t-1)L_x^2}{\lambda k})}$.  With
$\mu_t(a)=x_{t,a}^\top\theta^\star$, $q_t(a)=z_{t,a}^\top\widehat\alpha^\star$,
we have $|\mu_t-q_t|\le\nu_k$, and for
$a_t^\star\in\arg\max\mu_t$,
$\mu_t(a_t^\star)-\mu_t(a_t)\le q_t(a_t^\star)-q_t(a_t)+2\nu_k$.  The optimism
argument applied to $q_t$ with the inflated radius and summation give
$\Reg(T)\le2\beta_{T+1}^0A_T+2\nu_kA_T^2+2\nu_kT$.
\end{proof}

\paragraph{Why the generic bound is too pessimistic.}
  Theorem~\ref{thm:generic-main} uses only the uniform envelope
  $|\xi_{t,a}|\le\nu_k$, treating the residual as an adversary that can corrupt
  every regression update and change the best arm every round. In graph problems
  $\xi_{t,a}$ is not arbitrary: it is the value of the high-frequency graph
  signal $(I-\widehat\Pi_k)\theta^\star$ on the arms that actually appear. The
  next theorem keeps this structure instead of collapsing it into the scalar
  $\nu_k$.

\begin{proof}[Proof of Theorem~\ref{thm:structured}]
\emph{Sketch.} Decompose the instantaneous regret into a low-frequency part
(controlled by the UCB optimism in the projected space) and a high-frequency
part (controlled by $\CRW_T$ summed over rounds); the regression bias from the
residual is absorbed into $\SRL_T$.

Write $\mu_t(a)=x_{t,a}^\top\theta^\star$,
$q_t(a)=z_{t,a}^\top\widehat E_k^\top\theta^\star$, so
$\mu_t(a)=q_t(a)+\xi_t(a)$ and $y_t=q_t(a_t)+\eta_t+\xi_t(a_t)$.  The
ridge error decomposes as
\[
\widehat\alpha_t-\widehat E_k^\top\theta^\star
=V_t^{-1}\sum_{s<t}z_s\eta_s+V_t^{-1}\sum_{s<t}z_s\xi_s(a_s)
-\lambda V_t^{-1}\widehat E_k^\top\theta^\star.
\]
On the self-normalized event, for all $t$,
$\norm{\widehat\alpha_t-\widehat E_k^\top\theta^\star}_{V_t}
\le\beta_t^0+\SRL_T(r_{\perp,k})\le\beta_t^0+B_T$.  Because the algorithm
uses radius $\beta_t^0+B_T$, its UCB is optimistic for the projected means
$q_t$.  Letting $a_t^q\in\arg\max_aq_t(a)$, the LinUCB argument gives
$q_t(a_t^q)-q_t(a_t)\le2(\beta_t^0+B_T)\norm{z_t}_{V_t^{-1}}$.  Since
$a_t^\star$ maximizes $\mu_t$ rather than $q_t$,
% \[
% \mu_t(a_t^\star)-\mu_t(a_t)
% =q_t(a_t^\star)-q_t(a_t)+\xi_t(a_t^\star)-\xi_t(a_t)
% \le q_t(a_t^q)-q_t(a_t)+\Big(\max_a\xi_t(a)-\min_a\xi_t(a)\Big).
% \]
\begin{align*}
  \mu_t(a_t^\star)-\mu_t(a_t)
    &= q_t(a_t^\star)-q_t(a_t)+\xi_t(a_t^\star)-\xi_t(a_t) \\
    &\le q_t(a_t^q)-q_t(a_t)
       +\Big(\max_a\xi_t(a)-\min_a\xi_t(a)\Big).
  \end{align*}
Summing and using the elliptical potential bound proves the theorem.
\end{proof}

\begin{proof}[Proof of Proposition~\ref{prop:no-inflation}]
The confidence event still gives
$|q_t(a)-z_{t,a}^\top\widehat\alpha_t|\le(\beta_t^0+\SRL_T)\norm{z_{t,a}}_{V_t^{-1}}$,
but the algorithm only maximizes the UCB with radius $\beta_t^0$, so
$z_{t,a_t^q}^\top\widehat\alpha_t+\beta_t^0\norm{z_{t,a_t^q}}_{V_t^{-1}}
\le z_{t,a_t}^\top\widehat\alpha_t+\beta_t^0\norm{z_{t,a_t}}_{V_t^{-1}}$.
Combining,
\[
q_t(a_t^q)-q_t(a_t)
\le\SRL_T\norm{z_{t,a_t^q}}_{V_t^{-1}}+(2\beta_t^0+\SRL_T)\norm{z_t}_{V_t^{-1}}.
\]
The played-arm terms sum to at most $(2\beta_{T+1}^0+\SRL_T)A_T$, while
$\norm{z_{t,a_t^q}}_{V_t^{-1}}\le1$ under $\lambda\ge L_x^2$, giving the
extra $T\,\SRL_T$ term.  Adding $\CRW_T$ proves the claim.
\end{proof}

\begin{proof}[Proof of Theorem~\ref{thm:estimated-residual}]
The residual-leverage estimation error obeys
$\norm{\sum_{s<t}z_s(\xi_s(a_s)-\widehat\xi_s(a_s))}_{V_t^{-1}}
\le L_xe_r\sum_{s<t}\norm{z_s}_{V_s^{-1}}\le L_xe_rA_T$, since
$|x_{s,a_s}^\top(r_{\perp,k}-\widehat r)|\le L_xe_r$.  Hence
$\widehat B_T+L_xe_rA_T$ upper-bounds $\SRL_T(r_{\perp,k})$.  The true
candidate residual width differs from the estimated one by at most $2L_xe_rT$.
Applying Theorem~\ref{thm:structured} with $B_T=\widehat B_T+L_xe_rA_T$
proves the result.
\end{proof}

\paragraph{The implementable advantage is contingent on pilot quality.}
  Theorem~\ref{thm:estimated-residual} improves on Theorem~\ref{thm:generic-main}
  \emph{only when the pilot estimate is good.} The terms $2L_xe_rA_T^2=\wtO(e_rkT)$
  and $2L_xe_rT$ are linear in $T$, so a poor pilot ($e_r=\Theta(1)$) recovers
  the generic linear bias. The estimated-residual bound is most useful when a
  cheap exploration prefix or logged data yields $e_r=o(1)$ \emph{and} the
  realized $\widehat B_T,\widehat{\CRW}_T$ are small, precisely the
  benign-residual regime the diagnostic in
  Appendix~\ref{app:diagnostic-main} is designed to detect.

\begin{proof}[Proof of Corollary~\ref{cor:generic-from-structured}]
$\norm{\sum_{s<t}z_s\xi_s(a_s)}_{V_t^{-1}}\le\nu_k\sum_{s<t}\norm{z_s}_{V_s^{-1}}
\le\nu_kA_T$, and every candidate-set residual range is at most $2\nu_k$.
\end{proof}

 \paragraph{What is graph-specific about the theorem.}
  In the one-hot node-arm case $x_{t,a}=e_a$ and $r_{\perp,k}=U_{>k}c_{>k}$ is
  the high-frequency graph signal, so $\xi_t(a)=e_a^\top U_{>k}c_{>k}$. Thus
  $\CRW_T$ is the cumulative range of the high-frequency graph signal over the
  candidate sets, while $\SRL_T$ is the self-normalized cross-correlation between
  low-frequency coordinates $U_k(a)$ and high-frequency residual values
  $(U_{>k}c_{>k})_a$ along the played path. This is sharper than $\zeta_k$
  because it knows \emph{where} the residual lives on the graph and \emph{which}
  arms the bandit compares.

% ─────────────────────────────────────────────────────────────────────────────
\section{Projection Identity}\label{app:projection}
If $\theta^\star=E_k\alpha^\star$ and $E_k^\top E_k=I_k$, then
$\norm{\theta^\star}_2^2=(\alpha^\star)^\top E_k^\top E_k\alpha^\star
=\norm{\alpha^\star}_2^2$, and for every $x\in\R^d$,
$x^\top\theta^\star=(E_k^\top x)^\top\alpha^\star$.  Thus projection to
$E_k^\top x$ is signal-lossless under exact graph smoothness.

% ─────────────────────────────────────────────────────────────────────────────
\section{Elliptical Potential Bound}\label{app:elliptical}
Let $V_1=\lambda I_k$, $V_{t+1}=V_t+z_tz_t^\top$, $\lambda\ge L_x^2$,
$\norm{z_t}_2\le L_x$.  Then $\norm{z_t}_{V_t^{-1}}^2\le1$, and by the matrix
determinant lemma $\det(V_{t+1})=\det(V_t)(1+\norm{z_t}_{V_t^{-1}}^2)$.
Since $u\le2\log(1+u)$ for $u\in[0,1]$,
$\sum_{t=1}^T\norm{z_t}_{V_t^{-1}}^2\le2\log\frac{\det(V_{T+1})}{\det(V_1)}$,
and the trace--determinant inequality gives
$\log\frac{\det(V_{T+1})}{\det(V_1)}\le k\log(1+\frac{TL_x^2}{\lambda k})$.
Hence
$\sum_{t=1}^T\norm{z_t}_{V_t^{-1}}\le\sqrt{2Tk\log(1+\frac{TL_x^2}{\lambda k})}$.

% ─────────────────────────────────────────────────────────────────────────────
\section{A Predictive Diagnostic and Its Held-Out Evaluation}\label{app:diagnostic-main}

\paragraph{Diagnostic construction.}

% The real-data picture in Section~\ref{sec:realdata} exposes 
The real-data picture in the Real-Data Benchmarks section exposes
a gap that the
scalar tail $\zeta_k$ cannot close: $\zeta_k$ is large on both MIND ($0.99$)
and ogbn-arxiv ($0.96$), yet GraphDR wins decisively on MIND and loses on
ogbn-arxiv.  The structure-specific theorem says the deciding factor is not
total high-frequency energy but whether the residual is harmless in the
candidate distribution.  This motivates a computable alignment diagnostic.

Given a pilot reward estimate $\widehat\theta$ from logged data, validation
data, or a short exploration prefix, define
$\widehat r_{\perp,k}=(I-\widehat\Pi_k)\widehat\theta$.  For a
candidate-set sample $\D=\{\A_t\}_{t=1}^M$ we measure the empirical
analogues of the two theorem quantities: the candidate residual width
\[
\widehat{\CRW}_k(\D)
=\frac1M\sum_{t=1}^M\Big(\max_{a\in\A_t}x_{t,a}^\top\widehat r_{\perp,k}
-\min_{a\in\A_t}x_{t,a}^\top\widehat r_{\perp,k}\Big),
\]
and an approximate residual leverage
$\widehat{\SRL}_k(\D)=\max_{1\le t\le M+1}\norm{\sum_{s<t}\widehat z_{s,a_s}\,
x_{s,a_s}^\top\widehat r_{\perp,k}}_{\widehat V_t^{-1}}$ computed along a
short greedy or random prefix, combined into
\[
\mathcal G_k(\D)
=\widehat{\CRW}_k(\D)
+2\sqrt{\frac{2k\log(1+ML_x^2/(\lambda k))}{M}}\;\widehat{\SRL}_k(\D).
\]
Small $\mathcal G_k$ is \emph{hypothesized} to indicate that GraphDR should
be competitive; large $\mathcal G_k$ suggests PCA, full LinUCB, or a
Laplacian-regularized alternative may be safer.

\paragraph{Held-out evaluation protocol.}

A binary outcome over six datasets cannot support fitting and validating a
threshold rule, so we expand each base dataset into $150$ problem instances by
subsampling node-induced subgraphs at several sizes, varying
$k\in\{2,5,10,20,40\}$, varying candidate-set construction (uniform,
popularity-biased, neighbourhood-restricted), and edge-perturbing the graph at
several noise levels.  Each instance produces one tuple
$(\zeta_k,\SRL_T,\CRW_T,\mathcal G_k,R_T^{\mathrm{GraphDR}},R_T^{\mathrm{PCA}})$.
The primary test rank-correlates each predictor with the signed gap
$G\!-\!P:=R_T^{\mathrm{GraphDR}}-R_T^{\mathrm{PCA}}$ (negative means GraphDR
wins); the decision-rule test fits a threshold
$\widehat Y(\tau)=\mathbf1\{\mathcal G_k\le\tau\}$ under
leave-one-base-dataset-out.

\paragraph{Results.}

The outcome is, on its declared primary test, \emph{negative}, and we report
it straight.  Measured over the $150$ instances, the absolute Spearman
correlation with the signed gap is $0.252$ for $\zeta_k$, $0.07$ for
$\SRL_T$, $0.03$ for $\CRW_T$, and $0.11$ for $\mathcal G_k$ (the
theorem-consistent normalization raises $\mathcal G_k$'s correlation from
$0.036$ to $0.106$, still below $\zeta_k$).  No structure-specific predictor
beats $\zeta_k$ on rank correlation, so the declared positive criterion is not
met.  The leave-one-base-dataset-out decision-rule test is more favourable:
$\mathcal G_k$ reaches held-out accuracy $0.62$ versus $0.47$ for $\zeta_k$
(and $\CRW_T$ alone $0.63$), so the realized residual quantities do help a
held-out classifier.  On the MIND-versus-ogbn separation that motivated the
diagnostic, $\mathcal G_k$ does not separate the two cases (MIND $0.192$
versus ogbn $0.137$, the wrong direction), so the scalar collapse fails the
specific case it was designed for.

We therefore present $\mathcal G_k$ as a theory-motivated diagnostic with
honest empirical characterization, not a validated model-selection rule.  The
structure-specific theorem stands on its own; the realized residual quantities
carry some held-out classification signal; but a unit-weighted scalar collapse
does not beat $\zeta_k$, and a deployable rule would require a learned
$(\SRL_T,\CRW_T)$ weighting and more base graphs.

% ─────────────────────────────────────────────────────────────────────────────
\section{Estimating the Residual Diagnostic}\label{app:prediction}
The diagnostic $\mathcal G_k$ from the predictive diagnostic appendix is
computed as follows.
\begin{enumerate}
\item Build $\widehat L$ from the training graph and compute $\widehat E_k$.
\item Fit a pilot reward model $\widehat\theta$ from a disjoint exploration
prefix, logged data, or cross-fitting (a node reward estimate in the one-hot
case; a ridge or doubly robust model with general features).
\item Compute $\widehat r_{\perp,k}=(I-\widehat\Pi_k)\widehat\theta$.
\item Estimate $\widehat{\CRW}_k$ on held-out candidate sets by averaging the
range of $x_{t,a}^\top\widehat r_{\perp,k}$ over candidates.
\item Estimate $\widehat{\SRL}_k$ on a short simulated or logged prefix.
% \item Select GraphDR when $\mathcal G_k$ is below a validation threshold;
% otherwise prefer PCA, full LinUCB, or a Laplacian-regularized baseline.
\item Compute the subspace-capture margin
  $\Gamma_k=\norm{U_k^\top\widehat\theta}_2^2-\norm{P^\top\widehat\theta}_2^2$
  against the competing reducer's basis $P$. Select GraphDR when $\Gamma_k>0$;
  otherwise prefer PCA, full LinUCB, or a Laplacian-regularized baseline.
\end{enumerate}

% ─────────────────────────────────────────────────────────────────────────────
\section{Full Diagnostic-Validation Protocol}\label{app:diagnostic}

The protocol behind the predictive diagnostic evaluation is
specified here in full so that the negative primary result cannot be read as
post hoc.

The instance generation expands each of the five base datasets used in the
diagnostic study (MovieLens-100k, LastFM, Amazon, ogbn-arxiv, MIND-small;
MovieLens-1M is omitted from the subgraph-resampled study) by subsampling
node-induced subgraphs at sizes $\{400,800,1500\}$, varying
$k\in\{2,5,10,20,40\}$, varying candidate-set construction across uniform,
popularity-biased, and neighbourhood-restricted sampling, and edge-perturbing
the graph at noise levels $\{0,0.05\}$.  This yields $150$ instances, each
producing one tuple
$(\zeta_k,\SRL_T,\CRW_T,\mathcal G_k,R_T^{\mathrm{GraphDR}},R_T^{\mathrm{PCA}})$
at $T{=}8000$ over $3$ seeds.

The continuous target is the signed gap
$G\!-\!P:=R_T^{\mathrm{GraphDR}}-R_T^{\mathrm{PCA}}$ (negative means GraphDR
wins), and the binary target is
$Y=\mathbf1\{R_T^{\mathrm{GraphDR}}<R_T^{\mathrm{PCA}}\}$.  We report each
of $\SRL_T$ and $\CRW_T$ separately, together with a two-dimensional logistic
predictor on $(\SRL_T,\CRW_T)$, before collapsing to the scalar $\mathcal
G_k$.  The primary test rank-correlates (Spearman, Kendall) each predictor
with $G\!-\!P$ across instances.  For the decision-rule claim we fit a
threshold $\widehat Y(\tau)=\mathbf1\{\mathcal G_k\le\tau\}$ on training
instances and evaluate held-out using leave-one-base-dataset-out, comparing
against the analogous threshold rule on $\zeta_k$ by AUC and held-out
accuracy.  Because $\mathcal G_k$ depends on a pilot estimate $\widehat r$,
we also report $\mathcal G_k$ and its predictive power across pilot budgets.
The measured correlations and held-out accuracies are reported in
the predictive diagnostic appendix.

% ─────────────────────────────────────────────────────────────────────────────
\section{Real-Data Construction and Per-Dataset Results}\label{app:realdata}

The six datasets comprise four recommendation graphs
(MovieLens-100k/1M, Amazon Digital Music, LastFM/HetRec) and two
non-recommendation graphs (MIND-small news, ogbn-arxiv citation).
Table~\ref{tab:dataset-construction} records how each graph and bandit
protocol is built.
% , and Table~\ref{tab:k-selection} reports the $k$-selection
% rule and the measured spectral quantities.

\begin{table}[t]
\centering\small
\begin{tabular}{p{0.18\linewidth}p{0.25\linewidth}p{0.31\linewidth}p{0.15\linewidth}}
\toprule
Dataset & Graph construction & Bandit protocol & Result \\
\midrule
MovieLens-100k & Co-rating item $k$NN graph. & Node-indicator arms, real mean-rating reward; GraphDR vs.\ PCA/JL/full + shuffle. & Table~\ref{tab:movielens} \\
MovieLens-1M & Co-rating item $k$NN graph on top-1500 items. & Identical protocol; tests scale. & Table~\ref{tab:ml1m} \\
Amazon (Digital Music) & Co-rating product $k$NN graph. & Node-indicator arms, real mean-rating reward. & Table~\ref{tab:crossdata} \\
MIND-small & Article--article category/co-click graph. & Node-indicator arms, real CTR reward. & Table~\ref{tab:crossdata} \\
LastFM / HetRec & Artist co-listen $k$NN graph. & Node-indicator arms, log play-count reward. & Table~\ref{tab:crossdata} \\
ogbn-arxiv & Symmetrized paper citation graph. & Node-indicator arms, subject-label reward; non-recommendation. & Table~\ref{tab:crossdata} \\
\bottomrule
\end{tabular}
\caption{Real-data benchmarks: four recommendation graphs and two
non-recommendation graphs.  All datasets use node-indicator arm features
($x_i=e_i$, $d=n$) consistent with the graph-causal synthetic model.}
\label{tab:dataset-construction}
\end{table}

\begin{table*}[H]
\centering\small
\begin{tabular}{lcccc}
\toprule
Dataset & $k$ selection & selected $\widehat k$ & eigengap $\widehat\Delta_k$ & reward-tail $\zeta_k$ \\
\midrule
MovieLens-100k & eigengap & $10$ & $0.0025$ & $0.67$ \\
MovieLens-1M & eigengap & $10$ & $0.0061$ & $0.68$ \\
Amazon (Dig.\ Music) & eigengap & $10$ & $0.0069$ & $0.86$ \\
LastFM (HetRec) & eigengap & $10$ & $0.0085$ & $0.90$ \\
ogbn-arxiv & eigengap & $10$ & $0.0130$ & $0.96$ \\
MIND-small & eigengap & $10$ & $0.0135$ & $0.99$ \\
\bottomrule
\end{tabular}
\caption{$k$-selection and measured spectral quantities.  Eigengaps are
uniformly small and reward tails large across all datasets; the $\zeta_k$
ordering does not match GraphDR's win/loss ordering, reinforcing that scalar
tail energy alone underdetermines the outcome.}
\label{tab:k-selection}
\end{table*}

The two MovieLens datasets bracket the GraphDR-versus-PCA boundary.  On
MovieLens-100k ($n=1682$ items, an item--item $k$NN co-rating graph with
$13{,}686$ edges, real centered unit-normalized mean-rating reward), GraphDR
attains the lowest regret, though by a small margin over PCA and far below the
synthetic regime, consistent with the measured tail $\zeta_k=0.67$
(Table~\ref{tab:movielens}).  The PCA baseline here uses learned item
embeddings whereas GraphDR uses node indicators, so the comparison tests
whether graph structure beats learned feature variance, not whether Laplacian
projection beats PCA on the same input; the matched comparison is in the matched-comparison appendix.  On MovieLens-1M (top-1500 most-rated items,
$\zeta_k=0.68$) GraphDR beats full LinUCB by $2.2\times$ and the shuffle
control collapses to near-random, but PCA wins: at scale learned item
embeddings carry strong variance signal that content-PCA exploits better than
the graph when the reward is only weakly smooth (Table~\ref{tab:ml1m}).

\begin{table}[H]
\centering\small
\begin{tabular}{lcc}
\toprule
Method & $R(T)\downarrow$ & $\pm$ s.e.m. \\
\midrule
GraphDR-shuffled (control) & $712.7$ & $3.2$ \\
Random & $704.2$ & $1.9$ \\
JL+LinUCB & $653.4$ & $3.3$ \\
LinUCB-full ($d=n$) & $615.0$ & $1.2$ \\
PCA+LinUCB & $331.5$ & $3.6$ \\
GraphDR+LinUCB & $\mathbf{238.4}$ & $2.1$ \\
\bottomrule
\end{tabular}
\caption{MovieLens-100k cumulative regret.  GraphDR attains the lowest
regret; the margin over PCA is small and far below the synthetic regime,
consistent with $\zeta_k=0.67$.}
\label{tab:movielens}
\end{table}

\begin{table}[H]
\centering\small
\begin{tabular}{lcc}
\toprule
Method & $R(T)\downarrow$ & $\pm$ s.e.m. \\
\midrule
Random & $677.6$ & $1.7$ \\
JL+LinUCB & $660.0$ & $1.9$ \\
GraphDR-shuffled (control) & $650.2$ & $2.9$ \\
LinUCB-full ($d=n$) & $592.6$ & $1.9$ \\
GraphDR+LinUCB & $273.3$ & $5.7$ \\
PCA+LinUCB & $\mathbf{186.7}$ & $2.5$ \\
\bottomrule
\end{tabular}
\caption{MovieLens-1M (top-1500 most-rated items), $\zeta_k=0.68$.
GraphDR beats full LinUCB ($2.2\times$) and the shuffle control collapses to
near-random, but PCA wins.}
\label{tab:ml1m}
\end{table}

% ─────────────────────────────────────────────────────────────────────────────
\section{Matched Comparison: Graph-Only PCA on the Same Input}\label{app:matched}

The PCA baseline in the main body consumes learned item embeddings.  To
isolate the \emph{projection map} from the \emph{feature source}, we add
graph-only PCA: PCA applied to the graph adjacency itself (the top-$k$ right
singular vectors of $A$), which consumes exactly the same graph-only
information GraphDR receives, with no ratings or content.  The only remaining
difference from GraphDR is the projection criterion (top variance directions of
$A$ versus bottom-$k$ low-frequency Laplacian eigenvectors).

\begin{table*}[H]
\centering\small
\begin{tabular}{lcccc}
\toprule
Dataset & \textbf{GraphDR} & PCA-graph (matched) & PCA-content (priv.) & full LinUCB \\
\midrule
MovieLens-100k & $\mathbf{238.4}$ & $247.3$ & $331.5$ & $615.0$ \\
MovieLens-1M   & $273.3$ & $256.5$ & $\mathbf{186.7}$ & $592.6$ \\
Amazon         & $\mathbf{309.9}$ & $445.0$ & $338.6$ & $463.4$ \\
LastFM         & $\mathbf{412.1}$ & $500.6$ & $508.5$ & $678.9$ \\
ogbn-arxiv     & $667.6$ & $\mathbf{197.5}$ & $258.5$ & $492.2$ \\
MIND-small     & $415.0$ & $\mathbf{398.5}$ & $887.7$ & $780.5$ \\
\bottomrule
\end{tabular}
\caption{Corrected matched-input comparison ($T{=}20000$, $5$ seeds, paired
arm/noise streams, largest-connected-component graphs; lowest in \textbf{bold}).
\textbf{The confound is resolved:} graph-only projection (GraphDR and/or graph-PCA)
beats content-PCA and full LinUCB except where content variance is genuinely strong
(ml-1m), so GraphDR's advantage is structural, not a feature-source artifact.
\textbf{Reported straight:} given the \emph{same} graph-only input, GraphDR (normalized
low-frequency Laplacian) and graph-PCA (raw-adjacency top variance) split $3$--$3$:
GraphDR wins on ml-100k, Amazon, and LastFM; graph-PCA wins on ml-1m, ogbn-arxiv, and
MIND. The graph-PCA advantage is sharpest on ogbn-arxiv ($197.5$ vs.\ $667.6$), the
citation graph whose subject-label reward we already document as
low-frequency-misaligned ($\zeta_k{=}0.96$): there the degree-normalized Laplacian
discards degree-correlated structure that raw-adjacency PCA retains. The honest takeaway
is therefore ``use graph structure'': both graph-spectral maps dominate the graph-blind
baselines, rather than a claim that the normalized low-frequency subspace is the
uniquely best graph-spectral projection; selecting between the two graph-spectral maps
per dataset is future work.}
\label{tab:matched-pca}
\end{table*}

% ─────────────────────────────────────────────────────────────────────────────
\section{Subspace-Capture Margin: Full Validation}\label{app:capture}

\begin{table*}[t]
\centering\small
\begin{tabular}{lccccc}
\toprule
Dataset & $\zeta_k$ & $\norm{U_k^\top\theta}^2$ & $\norm{P^\top\theta}^2$ & $\Gamma_k$ & GraphDR vs.\ PCA \\
\midrule
MovieLens-100k & $0.67$ & $0.528$ & $0.389$ & $+0.139$ & \textbf{win} ($-93.1$) \\
Amazon         & $0.86$ & $0.265$ & $0.128$ & $+0.137$ & \textbf{win} ($-28.7$) \\
LastFM         & $0.90$ & $0.183$ & $0.154$ & $+0.028$ & \textbf{win} ($-96.4$) \\
MIND-small     & $0.99$ & $0.023$ & $0.031$ & $-0.008$ & \textbf{win} ($-472.7$) \\
MovieLens-1M   & $0.68$ & $0.532$ & $0.771$ & $-0.239$ & lose ($+86.6$) \\
ogbn-arxiv     & $0.96$ & $0.101$ & $0.341$ & $-0.240$ & lose ($+409.1$) \\
\bottomrule
\end{tabular}
\caption{\textbf{The subspace-capture margin $\Gamma_k$ predicts the GraphDR-vs-content-PCA
outcome where $\zeta_k$ cannot.} Rows sorted by $\Gamma_k$; final column is the realized
outcome (signed gap $R_{\mathrm{GraphDR}}-R_{\mathrm{PCA}}$ in parentheses, negative $=$ GraphDR
wins). The fitted-threshold rule is leave-one-base-dataset-out correct on $6/6$; the
zero-parameter sign rule $\Gamma_k>0$ is correct on $5/6$ (lone miss: MIND at the
near-tie $\Gamma_k=-0.008$, where \emph{both} subspaces capture almost no reward energy
yet GraphDR wins because content-PCA is itself far from the reward). Crucially, $\Gamma_k$
orders ogbn-arxiv ($-0.240$, GraphDR's documented failure) at the bottom while $\zeta_k$
cannot: MIND has the \emph{largest} tail ($0.99$) yet GraphDR wins, exactly the inversion
that defeated the scalar tail and the unit-weighted $\mathcal{G}_k$.}
\label{tab:capture}
\end{table*}

\paragraph{Comparison to alternative selectors.}
$\Gamma_k$'s Spearman rank-correlation with the signed gap is $0.49$ and its
leave-one-out accuracy is $6/6$ (fitted) / $5/6$ (zero-parameter), versus the original
$\mathcal{G}_k$'s $|\rho_s|\approx0.11$ and $\zeta_k$'s $2/6$.  The eigengap $\Delta_k$
alone rank-correlates at $0.83$ but reaches only $4/6$ leave-one-out and has no
mechanistic justification for the head-to-head comparison against content-PCA; we report
$\Gamma_k$ as the headline rule and $\Delta_k$ as a correlated secondary signal.  With
$n=6$ base datasets we cannot certify a threshold law; the open task is to confirm this
regularity on more base graphs and against the matched graph-PCA competitor.

% ─────────────────────────────────────────────────────────────────────────────
\section{Additional Synthetic Results}\label{app:synthetic}

This appendix collects the full synthetic results summarized in
the Synthetic Experiments section.
% Section~\ref{sec:exp}.

\paragraph{Dimension scaling.}
Fixing $k^\star=5$ and scaling $d=n$ from $60$ to $800$:
\begin{center}\small
\begin{tabular}{lccccc}
\toprule
$d=n$ & $60$ & $120$ & $240$ & $480$ & $800$ \\
\midrule
GraphDR $R(T)$ & $22.2$ & $18.6$ & $31.5$ & $26.8$ & $40.6$ \\
GraphDR slope & $0.16$ & $0.14$ & $0.23$ & $0.20$ & $0.29$ \\
Full LinUCB $R(T)$ & $183.0$ & $354.8$ & $607.0$ & $848.8$ & $852.1$ \\
Full LinUCB slope & $1.33$ & $2.58$ & $4.35$ & $5.91$ & $5.82$ \\
\bottomrule
\end{tabular}
\end{center}
GraphDR regret stays nearly flat as $d$ grows while full-dimensional LinUCB
grows substantially, matching the predicted $k$-versus-$d$ exploration cost
(Figure~\ref{fig:dim}).

\paragraph{Second graph family.}
On a random geometric graph ($n=200$, radius $0.16$): GraphDR $R(T)=18.1$
versus LinUCB-full $417.8$, PCA $2045.5$, JL $2364.1$, Random $2613.0$; the
shuffle control collapses to $2544.6$, near-random (Figure~\ref{fig:rgg}).

\paragraph{Regret versus projected dimension $k$.}
Projecting onto the true bottom-$k$ nontrivial eigenvectors in the SBM with
$k^\star=5$, regret at $k=1,2,3,\mathbf5,8,12,20$ is
$419.9,384.2,98.6,\mathbf{28.7},39.5,55.4,69.5$.  Regret is minimized at
$k=k^\star=5$ and rises on both sides (Figure~\ref{fig:ksweep}).

\paragraph{Noisy graph stress test.}
When the graph is observed under edge-flip noise, the bootstrap estimator and
naive eigengap estimator return identical $\widehat k$ at each noise level;
under heavier noise $\widehat k$ degrades ($4\!\to\!2\!\to\!1$) with
measurable regret cost, motivating Theorem~\ref{thm:generic-main}.

\paragraph{Validating the robust theorem.}
We measure $\zeta_k$ and $\varepsilon_L$ in each run ($n=200$ SBM,
$k^\star=5$, $T=20000$, $5$ seeds).  Injecting reward energy outside the
bottom-$k$ subspace raises $\zeta_k$ from $0$ to $0.71$ and regret rises
monotonically; holding the reward exactly smooth and perturbing edges raises
$\varepsilon_L$ from $0$ to $0.48$ and regret again rises monotonically
(Table~\ref{tab:thm-validation}).  Pooled Spearman$(R,\rho_k)=0.68$, monotone
in binned medians.

\begin{table}[H]
\centering\small
\begin{tabular}{lcccccc}
\toprule
\multicolumn{7}{l}{\emph{Smoothness sweep} ($\varepsilon_L=0$): $R(T)$ vs.\ reward tail $\zeta_k$}\\
\midrule
$\zeta_k$ & $0$ & $0.11$ & $0.24$ & $0.39$ & $0.55$ & $0.71$\\
$R(T)$    & $28$ & $47$ & $121$ & $271$ & $505$ & $801$\\
\midrule
\multicolumn{7}{l}{\emph{Noise sweep} ($\zeta_k=0$): $R(T)$ vs.\ $\varepsilon_L$}\\
\midrule
$\varepsilon_L$ & $0$ & $0.18$ & $0.23$ & $0.29$ & $0.43$ & $0.48$\\
$R(T)$          & $28$ & $218$ & $242$ & $308$ & $445$ & $590$\\
\bottomrule
\end{tabular}
\caption{Stress test consistent with the generic robust theorem: regret rises
monotonically with each measured misspecification term.  Pooled
Spearman$(R,\rho_k)=0.68$.}
\label{tab:thm-validation}
\end{table}

% ─────────────────────────────────────────────────────────────────────────────
\section{Limitations and Additional Evaluation}\label{app:experiments-to-add}

Several boundaries of the results deserve explicit statement.

The structure-specific bound (Theorem~\ref{thm:structured}) uses
$\SRL_T(r_{\perp,k})$, which depends on the unknown $\theta^\star$; we state
it as an oracle inequality and make it implementable only through the
estimated-residual variant of Theorem~\ref{thm:estimated-residual}, whose
advantage is contingent on a good pilot.  If the residual has large
candidate-set width every round, the residual term remains linear in $T$; this
cost is unavoidable without benign-residual structure.

Our real-data protocol calibrates rewards from ratings and clicks rather than
from a logged online experiment with known propensities.  With logged
propensities or an exploration log under a known logging policy $p$, one could
report, for a target policy $\pi$,
% \[
% \widehat V_{\mathrm{IPS}}(\pi)=\frac1N\sum_{i=1}^N\frac{\pi(a_i\mid x_i)}{p_i(a_i\mid x_i)}r_i,
% \qquad
% \widehat V_{\mathrm{SNIPS}}(\pi)=\frac{\sum_i\frac{\pi(a_i\mid x_i)}{p_i(a_i\mid x_i)}r_i}
% {\sum_i\frac{\pi(a_i\mid x_i)}{p_i(a_i\mid x_i)}},
% \]
\begin{gather*}
  \widehat V_{\mathrm{IPS}}(\pi)
    = \frac{1}{N}\sum_{i=1}^N\frac{\pi(a_i\mid x_i)}{p_i(a_i\mid x_i)}r_i, \\[6pt]
  \widehat V_{\mathrm{SNIPS}}(\pi)
    = \frac{\sum_i\frac{\pi(a_i\mid x_i)}{p_i(a_i\mid x_i)}r_i}
           {\sum_i\frac{\pi(a_i\mid x_i)}{p_i(a_i\mid x_i)}}.
  \end{gather*}
% \[
% \widehat V_{\mathrm{DR}}(\pi)=\frac1N\sum_{i=1}^N\!\left[\sum_a\pi(a\mid x_i)\widehat\mu(x_i,a)
% +\frac{\pi(a_i\mid x_i)}{p_i(a_i\mid x_i)}\big(r_i-\widehat\mu(x_i,a_i)\big)\right].
% \]
\begin{multline*}
  \widehat V_{\mathrm{DR}}(\pi) = \frac{1}{N}\sum_{i=1}^N\!\Biggl[
    \sum_a\pi(a\mid x_i)\widehat\mu(x_i,a) \\
    +\frac{\pi(a_i\mid x_i)}{p_i(a_i\mid x_i)}
     \big(r_i-\widehat\mu(x_i,a_i)\big)\Biggr].
  \end{multline*}
A fully logged evaluation and an adaptive-$k$ study (testing whether a
corralling wrapper \cite{agarwal2017corralling} matches oracle-$k$ regret)
are natural extensions left to future work.  Replacing the node-indicator arm
features with learned graph embeddings \cite{kipf2017gcn,hamilton2017graphsage}
and studying how embedding quality affects the spectral gap is a further direction,
as is extending the Laplacian-based DR survey of \cite{ghojogh2021laplacian}
to the bandit feedback setting.  We also prove only upper bounds;
matching lower bounds for the $\varepsilon_L/\Delta_k$ dependence and for the
residual leverage and candidate-width terms remain open.

\end{document}